\documentclass[lettersize,journal]{IEEEtran}
\usepackage{cite}
\usepackage{amsmath,amssymb,amsfonts}
\usepackage{graphicx}
\usepackage{textcomp}
\usepackage{xcolor}
\usepackage{amsmath, amssymb, epsfig, graphicx, bm}
\usepackage{cases}
\usepackage{float}
\usepackage{multirow}
\usepackage{graphicx}
\usepackage{color}
\usepackage{epstopdf}
\usepackage{enumitem}
\usepackage{amsthm}
\usepackage{amsfonts}
\usepackage{tcolorbox}
\usepackage{xcolor}
\usepackage{cite, url}
\usepackage{pifont}
\usepackage[all,cmtip]{xy}
\usepackage{color,hyperref}
\usepackage[noend]{algpseudocode}
\usepackage{algorithmicx,algorithm}
\usepackage{flushend}
\usepackage[normalem]{ulem} 

\newtheorem{theorem}{Theorem}
\newcommand{\E}{{\mathbb{E}}}

\newcommand{\R}{{\mathbb{R}}}

\newcommand{\B}{{\mathcal{B}}}

\newtheorem{Definition}{Definition}

\newtheorem{assumption}{Assumption}
\newtheorem{lemma}{Lemma}
\newtheorem{remark}{Remark}

\newcommand{\vx}{\bm{x}}

\newcommand{\N}{\mathcal{N}}

\newcommand{\cD}{\mathcal{D}}

\newcommand{\cS}{\mathcal{S}}
\newcommand{\cL}{\mathcal{L}}
\newcommand{\A}{\mathcal{A}}

\newcommand{\lp}{\left(}
\newcommand{\rp}{\right)}

\newcommand{\lnorm}{\left\|}
\newcommand{\rnorm}{\right\|}

\newcommand{\smallCE}[1]{$<$1e-07}

\begin{document}
\title{Generalization Error Analysis for \\ Attack-Free and Byzantine-Resilient Decentralized \\ Learning with Data Heterogeneity
}

\author{Haoxiang Ye, Tao Sun, and Qing Ling
\thanks{Haoxiang Ye and Qing Ling are with the School of Computer Science and Engineering, Sun Yat-Sen University, Guangzhou, Guangdong, China 510006.
        (e-mail addresses: yehx9@mail2.sysu.edu.cn; lingqing556@mail.sysu.edu.cn)}%
\thanks{Tao Sun is with the College of Computer, National
University of Defense Technology, Changsha, Hunan, China 410073.
(e-mail address: nudtsuntao@ 163.com).}
\thanks{Qing Ling (corresponding author) is supported in part by the National Key R\&D Program of China under Grant 2024YFA1014002; in part by the NSF China under Grant 62373388; and in part by the R\&D project of Pazhou Lab (Huangpu) under Grant 2023K0606. A short and preliminary version of this paper has been accepted by ICASSP 2025 \cite{ye2025noniid}.}
}



\maketitle

\begin{abstract}
Decentralized learning, which facilitates joint model training across geographically scattered agents, has gained significant attention in the field of signal and information processing in recent years. While the optimization errors of decentralized learning algorithms have been extensively studied, their genera- lization errors remain relatively under-explored. As the generali- zation errors reflect the scalability of trained models on unseen data and are crucial in determining the performance of trained models in real-world applications, understanding the generali- zation errors of decentralized learning is of paramount import- ance. In this paper, we present fine-grained generalization error analysis for both attack-free and Byzantine-resilient decentralized learning with heterogeneous data as well as under mild assump- tions, in contrast to prior studies that consider homogeneous data and/or rely on a stringent bounded stochastic gradient assump- tion. Our results shed light on the impact of data heterogeneity, model initialization and stochastic gradient noise -- factors that have not been closely investigated before -- on the generalization error of decentralized learning. We also reveal that Byzantine attacks performed by malicious agents largely affect the generali- zation error, and their negative impact is inherently linked to the data heterogeneity while remaining independent on the sample size. Numerical experiments on both convex and non-convex tasks are conducted to validate our theoretical findings.
\end{abstract}

\begin{IEEEkeywords}
Decentralized learning, Byzantine attacks, generalization error, data heterogeneity
\end{IEEEkeywords}


\section{Introduction}
\label{sec1}

\IEEEPARstart{R}{ecent} years have witnessed the significant advance of distributed learning, which enables geographically scattered devices to collaboratively train models, while ensuring the privacy of local data. According to the underlying network topologies, distributed learning can be classified into two cate- gories, federated learning and decentralized learning. Federat- ed learning relies on a central server to coordinate the learning process \cite{McMahan2016,Bian2024,9556559,10791812,zhang2025locally,shi2025optimal,zheng2025can}, while decentralized learning is able to operate autonomously without the need for a central server \cite{nedic2009distributed,Lian2017,li2025centralized,10852183,10847585,yuan2023removing,koloskova2020unified,9366365,10208129,9802673}. Notably, decentralized learning has gained increasing attention for its capacity to circumvent the communication bottleneck inherent in federated learning, caused by the central server.

Decentralized stochastic gradient descent (DSGD) has em- erged as one of the most widely adopted decentralized learning algorithms since its development in \cite{nedic2009distributed} and its application in training deep neural networks \cite{Lian2017}. In DSGD, the devices (also known as agents) execute local stochastic gradient descent steps according to their individual local datasets, while exchanging and aggregating the intermediate results with their neighboring agents to generate new iterates.

However, DSGD is highly vulnerable to unreliable agents who exchange incorrect messages due to computational faults or malicious attacks. Such behaviors are commonly modeled as Byzantine attacks and such agents are called as Byzantine agents \cite{lamport}. Under Byzantine attacks, an unknown number of Byzantine agents may deviate from the prescribed algorithmic protocol and transmit arbitrarily malicious messages to their neighbors, aiming to interfere the decentralized learning pro- cess. Even a single Byzantine agent can arbitrarily manipulate the iterates of its neighbors, and then the iterates of all agents in consequence.

To mitigate the impact of Byzantine attacks upon DSGD, a common strategy is to replace its weighted mean aggregation step with robust aggregation rules. A number of Byzantine-resilient decentralized learning algorithms, incorporating various robust aggregation rules, have been proposed and demonstrated effectiveness in defending against Byzantine attacks. Popular robust aggregation rules include trimmed mean (TM) \cite{fang2022bridge}, iterative outlier scissor (IOS) \cite{wu2022byzantine}, self-centered clipping (SCC) \cite{he2022byzantine}, etc. To achieve Byzantine-resilience and enable robust aggregation, TM removes the largest and the smallest values along each dimension of the received messages and averages the remaining values to obtain a robust estimation. IOS does not remove values dimension-wise; instead, it maintains a trust set and iteratively eliminates entire messages that deviate significantly from the trust set's average, followed by weighted averaging of the remaining messages. Rather than removing values or entire messages, SCC employs clipping, constraining the received messages that deviate significantly from the local model to mitigate the negative impact of Byzantine attacks. Besides, multiple robust aggregation rules can be combined; for example, \cite{yang2024byzantine} integrates IOS and SCC to devise remove-then-clip (RTC). It is also worth noting that \cite{wu2022byzantine} introduces a general framework to unify these robust aggregation rules, and \cite{ye2024tradeoff} subsequently analyzes their properties. These works have collectively demonstrated that, in the presence of Byzantine agents, the optimization error of Byzantine-resilient DSGD be- comes nonzero, even asymptotically.

\noindent\textbf{Generalization Error.}
Although numerous studies have esta- blished the optimization errors (also known as convergence analysis) associated with attack-free and Byzantine-resilient DSGD \cite{nedic2009distributed,Lian2017,li2025centralized,10852183,10847585,yuan2023removing,fang2022bridge,wu2022byzantine,he2022byzantine,yang2024byzantine,ye2024tradeoff,koloskova2020unified,9366365,10208129,9802673}, a critical aspect -- their generali- zation errors -- remains inadequately explored.
The generalization error, which reflects the scalability of a learning algorithm to unseen data, is crucial in determining the performance of a learned model in real-world applications. Extensive theoretical studies devote to analyzing the generalization errors of single-agent learning algorithms.
The generalization error bound of stochastic gradient descent (SGD) is established in \cite{hardt2016train}, with the tool of uniform stability, a framework originally introduced in \cite{bousquet2002stability}. Uniform stability measures the sensitivity of a learning algorithm's output when a single sample in the training dataset is replaced, and its connection with the generalization error is formalized in \cite{hardt2016train}. The more uniformly stable a learning algorithm is, the better the trained model tends to generalize. The lower bound of uniform stability of SGD, as shown in \cite{zhang2022stability}, aligns with the upper bound provided in \cite{hardt2016train}. The works of \cite{chen2024three,10934739} give the generalization error bounds for multi-objective learning and graph neural network training, respec- tively. To improve the generalization error bounds for SGD, the notion of on-average stability is proposed in \cite{kuzborskij2018data, lei2020fine}.


Although extending the generalization error analysis from single-agent learning to federated learning is natural, establishing the generalization error bounds for the decentralized learning algorithms is more challenging due to the discrepancy of the trained local models caused by the data heterogeneity. This challenge is overlooked in the prior works that extend the generalization error analysis from SGD to DSGD \cite{sun2021stability,deng2023stability,zhu2022topology,bars2023improved,ye2024ge,ye2024generalization}. Among them,  \cite{sun2021stability,deng2023stability,zhu2022topology} consider homogeneous data distributions, and \cite{sun2021stability,deng2023stability,bars2023improved,ye2024ge,ye2024generalization} make the stringent bounded stochastic gradient assumption that shelters the discrepancy caused by the data heterogeneity. Specifically, for attack-free DSGD, \cite{sun2021stability} provides the first analysis of the generalization errors of synchronous implementation, while \cite{deng2023stability} extends the results to asynchronous implementation. Further exploration of the impact of the underlying network topology on the generalization errors of DSGD is conducted in \cite{zhu2022topology}. The derived generalization error bounds are refined in \cite{bars2023improved}, matching those of single-agent learning algorithms \cite{hardt2016train}. On the other hand, for Byzantine-resilient DSGD, the works of \cite{ye2024ge,ye2024generalization} establish the generalization error bounds and reveal the negative impact brought by the Byzantine agents.

Nevertheless, all these existing theoretical results suppose that the local data distributions are identical across all agents (all non-Byzantine agents) and/or make the stringent bounded stochastic gradient assumption, overlooking the impact of data heterogeneity. In real-world applications, local data distributions often vary across the agents. The data heterogeneity is widely recognized as a significant challenge in decentralized learning \cite{hsieh2020non}; for example, it leads to a slow convergence rate of attack-free DSGD \cite{yuan2023removing} and a large convergence error of Byzantine-resilient DSGD \cite{wu2022byzantine}. To obtain deeper understanding of decentralized learning, an important question arises:

\begin{center}
\textbf{\textit{What is the impact of data heterogeneity on the generalization errors of decentralized learning algorithms?}}
\end{center}

While the recent work of \cite{sun2024understanding} has sought to answer the above question in the context of attack-free federated learning, their findings cannot be seamlessly extended to decentralized learning without any central server, especially in the presence of Byzantine agents.
More importantly, \cite{sun2024understanding} simultaneously assumes bounded stochastic gradients, bounded stochastic gradient noise and bounded data heterogeneity.
Although these assumptions simplify the analysis, assuming bounded stochastic gradients fails to capture the impact of data heterogeneity.
In addition, the assumption of bounded stochastic gradients is able to directly imply the other two, bringing redundancy. In contrast, our generalization error analysis removes the stringent bounded stochastic gradient assumption, allowing for a more accurate assessment of the impact of data heterogeneity.

\subsection{Contributions}
Our contributions are summarized as follows.
\begin{itemize}[leftmargin=*]

\item \textbf{Analysis of Generalization Error for Attack-free DSGD with Heterogeneous Data:}
For attack-free DSGD, we remove the stringent bounded stochastic gradient assumption present in the prior works and provide the first fine-grained generalization error analysis.
Our results provide novel in- sights into the influence of data heterogeneity, model initia- lization and stochastic gradient noise on the generalization error in decentralized learning, which has not been explored in the prior works.

\item \textbf{Analysis of Generalization Error for Byzantine-resilient DSGD with Heterogeneous Data:} Building on the results of attack-free DSGD, we establish the generalization error bounds for Byzantine-resilient DSGD with heterogeneous data.
    Our findings reveal that Byzantine attacks amplify the generalization error, and this impact is inherently linked to the data heterogeneity while remaining independent on the sample size. Furthermore, we provide the first analysis of the conditions under which cooperation remains beneficial under Byzantine attacks from a generalization perspective.

\item \textbf{Validation of Analysis with Numerical Experiments:} We conduct numerical experiments on strongly convex and non-convex tasks to validate our theoretical analysis, confirming the conclusion that higher data heterogeneity often corresponds to a larger generalization error. Our experimental results also demonstrate that Byzantine-resilient DSGD exhibits a larger generalization error compared to its attack-free counterpart, further supporting our theoretical findings.
\end{itemize}

Compared to the short and preliminary conference version \cite{ye2025noniid}, this paper has been extensively revised. The conference version solely analyzes the generalization error of attack-free DSGD with heterogeneous data. In contrast, this paper extends the scope to encompass Byzantine-resilient DSGD. Furthermore, we provide detailed proofs that address the challenges of analyzing the generalization errors of decentralized learning algorithms without relying on the stringent bounded gradients assumption.
Last but not the least, we enhance this paper with additional numerical experiments, such as Byzantine-resilient DSGD under various Byzantine attacks and non-convex tasks on the CIFAR-10 dataset. We also provide in-depth discussions on the numerical results.

\subsection{Paper Organization}
The rest of this paper is organized as follows. Section \ref{sec2} formulates the problem of decentralized learning and outlines the goal of generalization error analysis. Section \ref{sec-al} reviews a generic decentralized SGD framework, including both attack-free and Byzantine-resilient setups. In Section \ref{sec-pre}, we introduce the underlying assumptions and the technical tool of on-average stability.  The derived generalization error bounds for attack-free DSGD are presented in Section \ref{sec-ge}, while Section \ref{sec-ge-b} provides the generalization error bounds for Byzantine-resilient DSGD, along with a discussion on the cooperation gain from a generalization perspective.
Numerical experiments are presented in Section \ref{sec-num}, followed by concluding remarks in Section \ref{sec-con}.


\section{Problem Statement}
\label{sec2}

\subsection{Problem Setups}
Let us consider a fundamental learning task that involves collaboratively training a model with a group of agents over a decentralized network. In the \textit{attack-free} setup, this network is represented by an undirected and connected graph, $\mathcal{G}=(\mathcal{N},\mathcal{E})$, within which $\mathcal{N}$ stands for the set of $N = |\mathcal{N}|$ agents and $\mathcal{E}$ stands for the set of edges. If $e = (m, n) \in \mathcal{E}$, then the undirected communication link $e$ exists between the two agents $m$ and $n$, allowing them to exchange messages as neighbors. For agent $n$, we define its set of neighbors as $\mathcal{N}_n$. Each agent $n$ draws training samples from a local data distribution $\cD_n$, denoted as $\xi_n \sim \cD_n$. The objective is to find a global model $\vx^* \in \R^d$ that minimizes the population loss, defined as
\begin{equation}\label{1}
    F(\vx) := \frac{1}{N}\sum\limits_{n \in \N}\mathbb{E}_{\xi_n \sim \cD_n } f\left(\vx; \xi_n\right).
\end{equation}
Therein, $f\left(\vx; \xi_n\right)$ is the loss of $\vx \in \R^d$ evaluated on training sample $\xi_n$. With particular note, we allow for different local data distributions and will formally quantify the differences in Assumption \ref{assumption:heterogeneity}; in contrast, \cite{sun2021stability,deng2023stability,zhu2022topology} assume that all local data distributions are identical. The population loss cannot be directly evaluated since the local data distributions $\mathcal{D}_n$ are generally unknown. Therefore, a common alternative to \eqref{1} is to solve an approximate, empirical loss minimization problem over the union of the agents' local datasets.

Denote the union of the agents' local datasets to be $\cS := \cup_{n \in \mathcal{N}} \cS_{n}$, in which $\cS_n$ $=\{\xi_{n,1},\cdots,\xi_{n,Z}\}$ is the local dataset of agent $n \in \N$, consisting of $Z$ training samples independent- ly and identically drawn from $\mathcal{D}_n$. For simplicity, we assume that all local datasets are of the same size, but our results can be extended to the case where the local datasets vary in size. The resulting global empirical loss is  computed as the average of the local empirical losses across all agents, given by
\begin{align}\label{erm}
 F_{\cS}(\vx) \!&:=\! \frac{1}{N} \! \sum\limits_{n \in \N} F_{\mathcal{S}_n}\!\left(\vx\right)  \\~~ \text{with} ~~ F_{\mathcal{S}_n}\!\left(\vx\right) \!&:=\! \frac{1}{Z} \! \sum_{z=1}^{Z}f\left(\vx; \xi_{n,z}\right). \notag
\end{align}

In the \textit{Byzantine-resilient} setup, not all agents are reliable. When some agents are Byzantine, they shall disrupt other than contribute to the learning process. For notational convenience, in this setup we reuse $\mathcal{N}$ to denote the set of non-Byzantine agents with $N=|\mathcal{N}|$. Correspondingly, here we denote $\mathcal{B}$ as the set of Byzantine agents with $B=|\mathcal{B}|$. In addition to that all agents (including Byzantine and non-Byzantine) constitute an undirected and connected graph, the non-Byzantine agents must constitute an undirected and connected graph, too. Otherwise, the non-Byzantine agents would be unable to train a consensual model. We let $\mathcal{N}_n$ and $\mathcal{B}_n$ respectively represent the sets of non-Byzantine and Byzantine neighbors of agent $n$. The global population loss and global empirical loss are still respectively defined as in \eqref{1} and \eqref{erm}; however, it is important to emphasize that the contributions of Byzantine agents are excluded from these definitions.

For future usage, in both setups we denote $\vx^* \in \R^d $ and $\vx_{\cS}^*  \in \R^d$ as the minimizers of $F$ and $F_{\cS}$, respectively.




\subsection{Goal of Analysis}
Let us consider a decentralized learning algorithm $\mathcal{L}$ applied to $\cS$ and yielding an output $\mathcal{L}(\mathcal{S})$. Recall that $\cS$ is respectively standing for the unions of all and non-Byzantine agents' local datasets in the attack-free and Byzantine-resilient setups.
We define the output $\mathcal{L}(\mathcal{S})$ as $\bar\vx^k := \frac{1}{N}\sum_{n \in \N} \vx_n^k$, respectively standing for the average of all agents' local models at time $k$ in the attack-free setup, and the average of non-Byzantine agents' local models at time $k$ in the Byzantine-resilient setup. Such definitions align with the prior works \cite{sun2021stability,deng2023stability,zhu2022topology,bars2023improved,ye2024ge,ye2024generalization}.

The expected excess loss \cite{bottou2007tradeoffs} of the output $\mathcal{L}(\mathcal{S})$ can be decomposed into three terms, given by
  \begin{equation*}
  \begin{aligned}
  &\E_{\cS,\mathcal{L}}[F(\mathcal{L}(\cS))-F({\vx}^*)]=\underbrace{\E_{\cS,\mathcal{L}}[F(\mathcal{L}(\cS))-F_{\cS}(\mathcal{L}(\cS))]}_{\textrm{generalization  error}}\\
  &+\underbrace{\E_{\cS,\mathcal{L}}[F_{\cS}(\mathcal{L}(\cS))-F_{\cS}(\vx^*_{\cS})]}_{\textrm{optimization error}}+\underbrace{\E_{\cS,\mathcal{L}}[F_{\cS}(\vx^*_{\cS})-F({\vx}^*)]}_{\textrm{$\leq 0$}}.
  \end{aligned}
  \end{equation*}
The first term at the right-hand side, referred to as the generalization error, quantifies the performance difference of the output $\mathcal{L}(\cS)$ between the training samples and unseen testing samples. The second term, known as the optimization error, evaluates the performance of $\mathcal{L}(\cS)$ on the training samples. The last term is non-positive because: (i) $F_{\cS}({\vx}_{\cS}^*) \leq F_{\cS}({\vx}^*)$; and (ii)
$\E_{\cS,\mathcal{L}} [F({\vx}^*)] = \E_{\cS,\mathcal{L}} [F_{\cS}({\vx}^*)]$ as the training samples in $\cS_{n}$ are independently and identically drawn from $\cD_n$. To simplify the theoretical results, we let $F (\vx^*)=0$ as in \cite{dengstability,lei2020fine}, implying that the optimal model has sufficient capacity to perfectly fit the underlying data distribution.

Therefore, to bound the expected excess loss, we need to simultaneously investigate the optimization error and generalization error.
Although the optimization errors of attack-free and Byzantine-resilient DSGD have already been extensively studied \cite{nedic2009distributed,Lian2017,li2025centralized,10852183,10847585,yuan2023removing,fang2022bridge,wu2022byzantine,he2022byzantine,yang2024byzantine,ye2024tradeoff,koloskova2020unified,9366365,10208129,9802673}, the exploration into their generalization errors remains in the nascent stage \cite{sun2021stability,deng2023stability,zhu2022topology,bars2023improved,ye2024ge,ye2024generalization}. In this paper we contribute to this field by providing the first generali- zation guarantees for both attack-free and Byzantine-resilient DSGD with heterogeneous data, aiming to deepen the under- standing of how the data heterogeneity affects the generaliza- tion errors.

\section{Attack-free and Byzantine-resilient decentralized learning algorithms}
\label{sec-al}
\subsection{Attack-free DSGD}
To solve \eqref{erm}, we employ the widely-used DSGD algorithm outlined in Algorithm \ref{DSGD}. The procedure at time $k$ is broken down into two steps:
\begin{itemize}[leftmargin=*]
\item\textbf{Model Update:} Each agent $n$ uniformly and independently draws a training sample $\xi_n^{k}$ from the local dataset $\cS_n$, and computes a stochastic gradient $\nabla f(\vx^{k}_n; \xi_n^{k})$ based on the current local model $\vx^{k}_n \in \R^d$. Then, agent $n$ performs one stochastic gradient descent step to update its intermediate local model, given by $\vx_n^{k+\frac{1}{2}} = \vx^k_n - \alpha^{k} \nabla f(\vx^{k}_n; \xi_n^{k})$, where $\alpha^k > 0$ is a step size.

\item\textbf{Communication and Aggregation:} Each agent $n$ transmits its intermediate local model $\vx_n^{k+\frac{1}{2}}$ to all neighbors $m \in \mathcal{N}_n$. After receiving $\vx_{m}^{k+\frac{1}{2}}$ from all neighbors $m \in \mathcal{N}_n$, each agent $n$ performs an aggregation step using a proper mixing matrix $W=[w_{nm}]$ so as to update its local model, given by $\vx^{k+1}_n=  \sum_{m \in \mathcal{N}_n} w_{nm} \vx_{m}^{k+\frac{1}{2}}$.
\end{itemize}

The aggregation step relies on a mixing matrix $W=[w_{nm}]$. We require: (i) $w_{nm} \in [0,1]$; (ii) $\bm{1}^\top W = W\bm{1} = \bm{1} $, in which $\bm{1}$ is an $N$-dimensional all-one vector; and (iii) $w_{nm}>0$ if $e = (m, n) \in \mathcal{E}$ or $n=m$, and $w_{nm}=0$ otherwise. Such a network topology-dependent, doubly stochastic mixing matrix appears in various DSGD papers \cite{nedic2009distributed,Lian2017,li2025centralized,10852183,10847585,yuan2023removing,sun2021stability,deng2023stability,zhu2022topology,bars2023improved,koloskova2020unified,9366365,10208129,9802673}. For the mixing matrix $W$, we define $\lambda := 1-\|(I-\frac{1}{N}\bm{1}\bm{1}^\top)W\|^2 \in (0,1]$ for characterizing the network topology, with $\|\cdot\|$ being the matrix spectral norm.

\begin{algorithm}[t]
    \caption{Attack-free DSGD}
    \label{DSGD}
        \textbf{Input:} Initializations $\vx_n^0 \!=\!\vx^0$ for all $n \!\in\! \N$; step size $\alpha^k \!>\! 0$
        \begin{algorithmic}[1]
        \ForAll {$k = 0, 1, 2, \cdots$}
        \ForAll {agents $n \in \N$}
        \State Compute stochastic gradient $\nabla f(\vx^{k}_n; \xi_n^{k})$
        \State Compute $\vx_n^{k+\frac{1}{2}} = \vx^k_n - \alpha^{k} \nabla f(\vx^{k}_n; \xi_n^{k})$
        \State Send $\vx_n^{k+\frac{1}{2}}$ to all neighbors $m \in \N_n$
        \State Receive $\vx_{m}^{k+\frac{1}{2}}$ from all neighbors $m \in \N_n$
        \State Aggregate $\vx^{k+1}_n=  \sum_{m \in \mathcal{N}_n} w_{nm} \vx_{m}^{k+\frac{1}{2}} $
        \EndFor
        \EndFor
        \end{algorithmic}
\end{algorithm}

\subsection{Byzantine-resilient DSGD}
In the presence of Byzantine agents, we solve \eqref{erm} with the Byzantine-resilient variant of the DSGD algorithm, outlined in Algorithm \ref{robust-DSGD}. Under Byzantine attacks, the weighted averaging step in Line 7 of Algorithm \ref{DSGD} becomes highly vulnerable, since even a single Byzantine neighbor can arbitrarily manipulate the aggregation result of non-Byzantine agent $n \in \mathcal{N}$. A common approach to counteract the Byzantine attacks is to replace the weighted averaging step with a robust aggregation rule. We denote $\mathcal{A}_n$ as the robust aggregation rule employed by non-Byzantine agent $n \in \mathcal{N}$. In this paper, we require $\mathcal{A}_n$ to satisfy the following Definition \ref{definition:mixing-matrix} that is met by many popular robust aggregation rules. For the ease of presentation, we suppose that the non-Byzantine agents are indexed from $1$ to $N$ here and thereafter, though their identities are unknown to the algorithm in practice.


\begin{Definition}[Virtual mixing matrix and contraction constant corresponding to $\{\mathcal{A}_n\}_{n \in \N}$ \cite{wu2022byzantine}]
    \label{definition:mixing-matrix}
    Let us consider a matrix $W \!\in\! \mathbb{R}^{N \times N}$ whose $(n,m)$-th entry $w_{nm} \in [0, 1]$
    if $m \in \N_n \cup\{n\}$ and $w_{nm} = 0$ if $m \in \N$ but $m \notin \N_n \cup\{n\}$, for $n \in \N$.
    Further, $\sum_{m\in \N_n \cup \{n\}}w_{nm}=1$. Define $\hat \vx_n := \sum_{m \in \N_n\cup\{n\}}$ $w_{nm} \vx_{m,n}$. If there exists a constant $\rho \geq 0$ for any $n \in \N$ such that
    \begin{align}
        \label{inequality:robustness-of-aggregation-local}
              & \|\mathcal{A}_n (\vx_n, \{\vx_{m,n}\}_{m\in \N_n\cup \mathcal{B}_n} )-\hat \vx_n \| \\
         \leq & \rho \max_{m \in \N_n \cup \{n\}}\|\vx_{m} - \hat \vx_n\|, \notag
    \end{align}
    then $W$ is the virtual mixing matrix and $\rho$ is the contraction constant associated with the robust aggregation rules $\{\mathcal{A}_n\}_{n\in \N}$.
\end{Definition}

Intuitively, a robust aggregation rule $\mathcal{A}_n$ employed by non- Byzantine agent $n \in \N$ should be able to effectively approxi- mate a ``proper'' weighted average of local models from agents $m \in \N_n \cup \{n\}$. Definition \ref{definition:mixing-matrix} describes such an approximation ability utilizing the associated virtual mixing matrix and the contraction constant. Our previous work of \cite{ye2024tradeoff} has already proven that several popular robust aggregation rules, including TM, IOS and SCC, satisfy Definition \ref{definition:mixing-matrix}. Their virtual mixing matrices and contraction constants have also been analyzed in \cite{ye2024tradeoff}. In contrast to the doubly stochastic mixing matrix in attack-free DSGD, the virtual mixing matrices associated with these robust aggregation rules are only guaranteed to be row stochastic, whereas the contraction constants are typically nonzero. We introduce $\chi^2 := \frac{1}{N}\|W^\top \bm{1}- \bm{1} \|^2$ to quantify the deviation of $W$ from being doubly stochastic, with $\bm{1}$ being an $N$-dimensional all-one vector. For the virtual mixing matrix, we accordingly define $\lambda := 1-\|(I-\frac{1}{N}\bm{1}\bm{1}^\top)W\|^2 \in (0,1]$. Note that similar definitions have also appeared in other works to characterize a general class of robust aggregation rules; see \cite{he2022byzantine, kuwaranancharoen2025geometric} for examples.

\begin{algorithm}[t]
    \caption{Byzantine-Resilient DSGD}
    \label{robust-DSGD}
        \textbf{Input:} Initializations $\vx_n^0\!=\!\vx^0$ for all $n \!\in\! \N$; step size $\alpha^k \!>\! 0$
        \begin{algorithmic}[1]
        \ForAll {$k = 0, 1, 2, \cdots$}
        \ForAll {non-Byzantine agents $n \in \N$}
        \State Compute stochastic gradient $\nabla f(\vx^{k}_n; \xi_n^{k})$
        \State Compute $\vx_n^{k+\frac{1}{2}} = \vx^k_n - \alpha^{k} \nabla f(\vx^{k}_n; \xi_n^{k})$
        \State Send $\vx_{n,m}^{k+\frac{1}{2}} = \vx_n^{k+\frac{1}{2}}$ to all neighbors $m \in \N_n\cup \B_n$
        \State Receive $\vx_{m,n}^{k+\frac{1}{2}}$ from all neighbors $m \in \N_n\cup \B_n$
        \State Aggregate $\vx^{k+1}_n= \mathcal{A}_n (\vx_n^{k+\frac{1}{2}}, \{ \vx_{m,n}^{k+\frac{1}{2}}\}_{m\in \N_n\cup \B_n})$
        \EndFor
        \ForAll {Byzantine agents $n\in \B$}
        \State Send $\vx^{k+\frac{1}{2}}_{n,m}=*$ to all neighbors $m \in \N_n\cup \B_n$
        \EndFor
        \EndFor
        \end{algorithmic}
\end{algorithm}

\section{Preliminaries}
\label{sec-pre}
\subsection{Assumptions}
We make the following assumptions on the loss.


\begin{assumption}[Strong Convexity]  \label{assumption:convex}
	The loss $f(\vx; \xi)$ is $\mu$-str- ongly convex for every $\xi$, i.e., for any $\vx$ and $\bm{y} \in \R^d$, $f(\vx; \xi) \geq f(\bm{y}; \xi) + \langle \vx - \bm{y}, \nabla f(\bm{y}; \xi) \rangle + \frac{\mu}{2} \|\vx-\bm{y}\|^2 $.
\end{assumption}

\begin{assumption}[Smoothness]  \label{assumption:Lip}
 The loss $f(\vx; \xi)$ is $L$-smooth for every $\xi$, i.e., for any $\vx$ and $\bm{y} \in \R^d$, $\| \nabla f(\vx; \xi) - \nabla f(\bm{y}; \xi) \| \leq L \|  \vx - \bm{y}\|$.
\end{assumption}

\begin{assumption}[Bounded Stochastic Gradient Noise]
	\label{assumption:variance}
	There exists $\sigma^2 > 0$ such that $\E_{\xi_{n,z}} \|\nabla f(\vx;\xi_{n,z}) - \nabla F_{\cS_n}(\vx)\|^2 \leq \sigma^2$, for any $\vx \in \R^d$ and any agent $n \in \N$ in both attack-free and Byzantine-resilient setups.
\end{assumption}


\begin{assumption}[Bounded Heterogeneity]
	\label{assumption:heterogeneity}
	There exists $\delta^2 > 0$ such that for any $\vx \! \in \! \R^d$, $ \frac{1}{N}\sum_{n \in \N} \|\nabla F_{\cS_n}(\vx) - \nabla F_{\cS}(\vx)\|^2 $ $\leq \! \delta^2$ in both attack-free and Byzantine-resilient setups.
\end{assumption}

Assumptions \ref{assumption:convex} and \ref{assumption:Lip} are standard in analyzing the generalization errors of decentralized learning algorithms \cite{sun2021stability,deng2023stability,bars2023improved,ye2024ge,ye2024generalization}. Nevertheless, these works all make an additional assumption of bounded stochastic gradients, i.e., $\| \nabla f(\vx;\xi_{n,z})\|$ $\leq M$. Although such an additional assumption simplifies the analysis, it is stronger than Assumptions \ref{assumption:variance} and \ref{assumption:heterogeneity}, and more importantly, overlooks the influence of data heterogeneity -- please refer to Remark \ref{remark-3} for more details. In contrast, Assumptions \ref{assumption:variance} and \ref{assumption:heterogeneity} are widely adopted in the analysis of learning algorithms and commonly satisfied in various applications\cite{Bian2024,shi2025optimal,Lian2017,yuan2023removing,koloskova2020unified,9366365,9802673,wu2022byzantine,he2022byzantine,yang2024byzantine,kuzborskij2018data, lei2020fine}.


\subsection{On-average Stability and Generalization Error}

To remove the bounded stochastic gradient assumption and account for the impact of data heterogeneity on the generalization error, the traditional tool of uniform stability \cite{hardt2016train,bousquet2002stability} is insufficient. Instead, we use a new tool of on-average stability introduced in \cite{kuzborskij2018data, lei2020fine}. Note that on-average stability has already been used in \cite{zhu2022topology,bars2023improved} to analyze decentralized learning algorithms, but \cite{zhu2022topology} investigates homogeneous data and \cite{bars2023improved} needs the bounded stochastic gradient assumption.

\begin{Definition}[On-average Stability \cite{lei2020fine,zhu2022topology,bars2023improved}]
\label{def-sta}
Let the union $\cS$ = $\cup_{n \in \mathcal{N}} \cS_{n}$ with $\cS_n$ $=\{\xi_{n,1},\cdots,\xi_{n,Z}\}$ consisting of the training samples independently and identically drawn from the local data distribution $\cD_n$.
For any $n \in \mathcal{N}$ and $z \in \{1,\cdots,Z\}$, denote $\cS^{(n,z)} = \{\cS_{m} | m \in \mathcal{N},  m \neq n\} \cup \cS_n^{(z)}$ where $\cS_n^{(z)} = \{ \xi_{n,1},$ $\cdots,\xi_{n,z-1}, \xi'_{n,z}, \xi_{n,z+1},\cdots,\xi_{n,Z}\}$, the dataset formed from $\cS$ via replacing the $z$-th element of the $n$-th
agent's dataset by a training sample $\xi'_{n,z}$ drawn from $\cD_n$. A stochastic algorithm $\mathcal{L}$ is on-average $\epsilon$-stable if
\begin{equation*}
\frac{1}{NZ} \sum_{n \in \mathcal{N}}\sum_{z=1}^{Z} \E_{\cS,\cS^{(n,z)},\mathcal{L}} \left[ \|\cL(\cS) - \cL(\cS^{(n,z)})\|^2 \right] \le \epsilon,
\end{equation*}
where the expectation $\E_{\cS,\cS^{(n,z)},\mathcal{L}}$ is taken over the randomness of $\cS$, $\cS^{(n,z)}$ and $\mathcal{L}$.
\end{Definition}

\begin{remark}
Let us briefly recall the definition of uniform stability and explain why on-average stability can better capture the impact of data heterogeneity.
We call a stochastic algori- thm $\mathcal{L}$ to be $\epsilon$-uniformly stable if
$\sup_{\xi} \E_{\cS,\cS^{(n,z)},\mathcal{L}} [ f(\mathcal{L}(\cS); \xi) - f(\mathcal{L}(\cS^{(n,z)}); \xi) ] \le \epsilon$. Observe that on-average stability measures the \textbf{average} sensitivity on the training samples, while uniform stability considers the \textbf{worst-case} sensitivity. In the context of decentralized learning, the advantage of on-average stability lies in traversing and averaging across different agents, and hence can elaborately characterize the influence from data heterogeneity.
\end{remark}

\begin{remark}
This paper investigates the generalization abilities of decentralized learning algorithms utilizing the tool of on-average stability under perturbed training samples, as defined in Definition \ref{def-sta}.
This approach falls into the broader category of perturbation-based analysis, which has been extensively applied in various signal and information processing problems.

For example, small-signal stability analysis examines how a power system responds to uncertainties in wind generation, ensuring stability around equilibrium points \cite{7946176}. In robust estimation, perturbation-based analysis is used to evaluate an estimator's sensitivity to system modeling uncertainties \cite{wang2021robust}. Similarly, in sparse signal recovery, such analysis provides insights for guaranteeing accurate reconstruction in the presence of perturbations and outliers \cite{10041941,6204357}.

\end{remark}

If a stochastic algorithm $\mathcal{L}$ is on-average stable, the trained model is not sensitive to a small perturbation upon the training dataset and exhibits a low generalization error. In the following lemma, we give the connection between on-average stability and generalization error in decentralized learning.

\begin{lemma}[Generalization Error through On-average Stability \cite{lei2020fine}]
\label{lemma-sta}
If a decentralized learning algorithm $\mathcal{L}$ is on-average $\epsilon$-stable and Assumption \ref{assumption:Lip} holds, then for any $v>0$, its gene- ralization error satisfies
\begin{align}
 \label{lemma}
 & | \E_{\cS,\mathcal{L}}[F(\mathcal{L}(\cS))-F_{\cS}(\mathcal{L}(\cS))] |  \\
 \leq & \frac{v+L}{ 2 N Z} \sum_{n \in \mathcal{N}}\sum_{z=1}^{Z} \E_{\cS,\cS^{(n,z)},\mathcal{L}} \|\cL(\cS) - \cL(\cS^{(n,z)})\|^2  \nonumber \\
 & + \frac{L}{v} \E_{\cS,\mathcal{L}} F_{\cS}(\mathcal{L}(\cS)). \nonumber
\end{align}
\end{lemma}

Lemma~\ref{lemma-sta} demonstrates that the generalization error can be bounded by the sum of two terms in \eqref{lemma}.
The first term at the right-hand side of \eqref{lemma} corresponds to the on-average stability defined in Definition~\ref{def-sta} while the second term, $\E_{\cS,\mathcal{L}} F_{\cS}(\mathcal{L}(\cS))$ represents the expected function value, which can be further bounded with the optimization error because $F(\vx^*)=0$ by hypothesis. Thus, to bound the generalization error, we need to analyze the on-average stability and the optimization error.

\begin{remark}
\label{remark-3}
    When alternatively assuming bounded stochastic gradients, i.e.,  $\| \nabla f(\vx;\xi_{n,z})\| \leq M$, we are able to directly derive $| \E_{\cS,\mathcal{L}}[F(\mathcal{L}(\cS))-F_{\cS}(\mathcal{L}(\cS))] | \leq M \E_{\cS,\cS^{(n,z)},\mathcal{L}} \|\cL(\cS) - \cL(\cS^{(n,z)})\|$
    as discussed in  \cite{sun2021stability,deng2023stability,bars2023improved,ye2024ge,ye2024generalization}. Although significantly simplifying the analysis, the bounded stochastic gradient assumption is strong and fails to accurately capture the impact of data heterogeneity. In fact, it immediately implies Assumptions \ref{assumption:variance} and \ref{assumption:heterogeneity}: (i) $\E_{\xi_{n,z}} \|\nabla f(\vx;\xi_{n,z}) - \nabla F_{\cS_n}(\vx)\|^2 \leq $ $2 \E_{\xi_{n,z}} \|\nabla f(\vx;\xi_{n,z})\|^2 \!+\! 2 \| \nabla F_{\cS_n}(\vx)\|^2 \leq 4 M^2 $; (ii) $ \frac{1}{N}\sum_{n \in \mathcal{N}} $ $\|\nabla F_{\cS_n}(\vx) \! - \! \nabla F_{\cS}(\vx)\|^2 \!\leq\! \frac{2}{N}\sum_{n \in \mathcal{N}} \|\nabla F_{\cS_n}(\vx) \|^2 \!+ \! \frac{2}{N}\sum_{n \in \mathcal{N}} $ $\| \nabla F_{\cS}(\vx)\|^2 \leq 4 M^2 $.
\end{remark}

\begin{remark}
   In this paper, our goal is to conduct generalization analysis of decentralized learning algorithms under data heterogeneity without assuming bounded stochastic gradients, primarily focusing on strongly convex losses. Our analysis can be readily extended to convex losses through replacing the $(1 - \alpha^k \mu)$-expansive property used in the current analysis with the $1$-expansive property during selecting different training samples; see \eqref{g-6}--\eqref{g-7} for the attack-free setup, as well as \eqref{ger-7}--\eqref{ger-9} and \eqref{gb-6}--\eqref{gb-8} for the Byzantine-resilient setup. We omit the extension due to the page limitation.

   For nonconvex losses under the additional assumption of bounded stochastic gradients, the generalization error analysis has been conducted in \cite{sun2021stability,deng2023stability,bars2023improved,ye2024ge,ye2024generalization}. However, without this assumption, the analysis is significantly more challenging. As shown in Lemma \ref{lemma-sta}, the generalization error depends on the function value $\E_{\cS,\mathcal{L}} F_{\cS}(\mathcal{L}(\cS))$, which is difficult to bound for nonconvex losses. In contrast, for strongly convex and convex losses, the function value $\E_{\cS,\mathcal{L}} F_{\cS}(\mathcal{L}(\cS))$ -- a standard measure of the learning process -- can be effectively bounded. Consequently, \cite{lei2023stability} imposes the Polyak-Lojasiewicz (PL) condition to confine the nonconvex losses.
   Investigating the generalization error of Byzantine-resilient DSGD for general nonconvex losses remains an open question, which we leave for future work.
\end{remark}

\section{Generalization Error Analysis For Attack-Free DSGD}
\label{sec-ge}

In this section, we analyze on-average stability, optimization error and generalization error of attack-free DSGD with heterogeneous data.

\subsection{On-average Stability of Attack-free DSGD}
With the aid of Lemma \ref{lemma-sta}, we are able to explore the genera- lization error of Algorithm \ref{DSGD} with heterogeneous data through leveraging the concept of on-average stability. We present the derived on-average stability bound in the following theorem.

\begin{theorem}[On-average Stability of Attack-free DSGD with Heterogeneous Data]
\label{the-onsta}
Under Assumptions \ref{assumption:convex}--\ref{assumption:heterogeneity}, if we set a proper step size $\alpha^k= \frac{1}{\mu(k+k_0)} $, where $k_0$ is sufficiently large, then at any given time $k$, the on-average stability of Algorithm \ref{DSGD} is bounded by
\begin{align}
        \label{thm-on}
        & \frac{1}{NZ} \sum_{n \in \mathcal{N}}\sum_{z=1}^{Z} \E_{\cS,\cS^{(n,z)},\mathcal{L}} \left[ \|\cL(\cS) - \cL(\cS^{(n,z)})\|^2 \right]  \\
        \leq &   \frac{ 4  \sum_{k'=0}^{k-1} \sum_{n \in \mathcal{N}} \sum_{z=1}^{Z}  \lp \E_{\cS,\mathcal{L}} \| \nabla f(\vx^{k'}_n; \xi_{n,z}) \|^2 / (k'+k_0) \rp }{\mu^2 N^2 Z^2 (k+k_0-1)}  \nonumber  \\
       & +   \frac{ 4   \sum_{k'=0}^{k-1} \sum_{n \in \mathcal{N}} \sum_{z=1}^{Z}    \E_{\cS,\mathcal{L}} \| \nabla f(\vx^{k'}_n; \xi_{n,z})\| ^2 }{\mu^2  N^2  Z^3 (k+k_0-1)}.  \nonumber
\end{align}
\end{theorem}
Using the decaying step size $\alpha^k = \frac{1}{\mu(k + k_0)}$, in which $k_0$ is introduced to prevent the initial step size from being too large, Theorem \ref{the-onsta} establishes the on-average stability bound of Algorithm \ref{DSGD}. According to Theorem \ref{the-onsta}, the on-average stability is determined by the accumulation of the squared stochastic gradient norms $\| \nabla f(\vx^{k'}_n; \xi_{n,z})\|^2$, demonstrating that smaller stochastic gradient norms lead to better on-average stability. Therefore, if we further assume bounded stochastic gradients, i.e., $\| \nabla f(\vx^{k'}_n; \xi_{n,z})\| ^2 \!\leq\! M^2$, the bound in \eqref{thm-on} will reduce to $O(\frac{M^2}{NZ})$, aligning with the results in \cite{bars2023improved}. Rather than assuming a fixed upper bound $M$ and disregarding the dynamics of the stochastic gradient norm, we observe from Theorem \ref{the-onsta} that the on-average stability is inherently linked to the learning process, as the stochastic gradient norm tends to decrease over the course of training. Such a link also appears in single-agent learning \cite{lei2020fine, lei2023stability}. But in decentralized learning, accounting for the dynamics of the stochastic gradient norm and its inf- luence on the on-average stability further enables us to under- stand the role of data heterogeneity in the generalization error, as will be formally shown in Theorem \ref{the-ge}.

\subsection{Optimization Error of Attack-free DSGD}

Although the optimization error of attack-free DSGD with heterogeneous data has been extensively studied, for completeness, we provide an upper bound in the following theorem.

\begin{theorem}[Optimization Error of Attack-free DSGD with Heterogeneous Data \cite{yuan2023removing}]
\label{the-oe}
Under Assumptions \ref{assumption:convex}--\ref{assumption:heterogeneity}, if we set a proper step size $\alpha^k= \frac{1}{\mu(k+k_0)} $, where $k_0$ is sufficiently large, then at any given time $k$, the optimization error of Algorithm \ref{DSGD} is bounded by
\begin{align}
        \label{thm-oe}
        &   \E_{\cS,\mathcal{L}} [F_{\cS}(\bar\vx^{k}) - F_{\cS} (\vx_{\cS}^*)]
\leq    \frac{L(k_0-1)}{2(k+k_0-1)} \|\bar \vx^0 - \vx_{\cS}^* \|^2  \\
& +    \frac{L\sigma^2 ln (k+k_0-1)}{2\mu^2 N (k+k_0-1)}
 +   \frac{ C_1 (\sigma^2 + \delta^2)}{\mu^3 (k+k_0-1)}.  \nonumber
\end{align}
Here $C_1>0$ is a constant.
\end{theorem}

Theorem \ref{the-oe} shows that the optimization error of attack-free DSGD with heterogeneous data is influenced by the model initialization, stochastic gradient noise and data heterogeneity. With sufficient training time, the optimization error converges to zero.
The convergence rate is in the order $\tilde O(\frac{1}{k})$ (with $\tilde O(\cdot)$ hiding logarithm factors), which aligns with the analysis in \cite{koloskova2020unified, deng2023stability,sun2021stability}.

\subsection{Generalization Error of Attack-free DSGD}
\label{sec:5c}
Applying Lemma \ref{lemma-sta} in conjunction with Theorems \ref{the-onsta} and \ref{the-oe}, we can now formally derive the generalization error bound.


\begin{theorem}[Generalization Error of Attack-free DSGD with Heterogeneous Data]
\label{the-ge}
When $k \asymp \mu N Z$, under Assumptions \ref{assumption:convex}--\ref{assumption:heterogeneity}, if we set a proper step size $\alpha^k= \frac{1}{\mu(k+k_0)}$, where $k_0$ is sufficiently large, the generalization error of Algorithm \ref{DSGD} is bounded by
\begin{align}
        \label{thm-ge-1}
        & \E_{\cS,\mathcal{L}}[F(\bar \vx^k)-F_{\cS}(\bar \vx^k))] \\
        \leq &  O(\frac{\|\bar\vx^0 - \vx_{\cS}^*\|^2}{\mu N Z})+  \tilde O(\frac{\sigma^2}{\mu N Z})  + O(\frac{\delta^2}{\mu N Z}). \nonumber
\end{align}
\end{theorem}

Theorem \ref{the-ge} establishes the generalization error bound for Algorithm \ref{DSGD} with heterogeneous data.
We can observe that the generalization error of Algorithm \ref{DSGD} is influenced by several factors: the initialization quality $\|\bar\vx^0 -\vx_{\cS}^*\|^2$, the stochastic gradient noise variance $\sigma^2$ and the data heterogeneity $\delta^2$.

The first term at the right-hand side of \eqref{thm-ge-1} highlights the significance of initialization. When the algorithm begins with a more favorable initial model, it tends to be more stable and generalizes better. In the case of $\bar\vx^0 = \vx_{\cS}^*$ and $\sigma^2=\delta^2=0$, Algorithm \ref{DSGD} does not update in expectation, resulting in per- fect stability with zero generalization error. If $\bar\vx^0 = \vx_{\cS}^*$, $\sigma^2=0$ but $\delta^2\neq 0 $, Algorithm \ref{DSGD} still updates and the generalization error is nonzero since $\E_{\cS,\cL} \|\nabla f(\vx_n^0, \xi_n^0)\| \neq 0$, reflecting the influence from data heterogeneity. Such dependence on the initial model aligns with that of single-agent algorithms \cite{kuzborskij2018data}.

The second and third terms at the right-hand side of \eqref{thm-ge-1} underscore the importance of stochastic gradient noise and data heterogeneity, respectively. When $\sigma^2=0$ such that the full gradients are used at each time, the second term vanishes as also observed in single-agent algorithms \cite{kuzborskij2018data}. When $\delta^2=0$, meaning that all local data distributions are essentially identi- cal, the third term disappears. However, when $\delta^2$ is large and data heterogeneity is remarkable, the generalization ability of attack-free DSGD deteriorates. This result gives the first-ever explanation for the poor performance of attack-free DSGD when faced with highly heterogeneous data \cite{hsieh2020non}, from the generalization perspective.

\subsection{Comparisons with Existing Generalization Analysis}

Comparing Theorem \ref{the-ge} with the existing generalization error analysis in \cite{sun2021stability,deng2023stability,bars2023improved,ye2024ge,ye2024generalization}, we can observe that our bound is tighter and offers new insights under relaxed assumptions, as summarized in Table \ref{tab-com}. First, the aforementioned works rely on the assumption of bounded stochastic gradients, and do not necessarily account for data heterogeneity. Second, our bound is tighter than those in \cite{sun2021stability,deng2023stability}, which are both in the order of $O(\frac{M^2}{\mu Z N})+O(\frac{1}{\mu})$. Our bound vanishes as the sample size $ZN$ tends to infinity, while the bounds in \cite{sun2021stability,deng2023stability} do not due to additional error terms. Third, our bound offers new insights compared to those in \cite{bars2023improved,ye2024ge,ye2024generalization}, which are both in the order of $O(\frac{M^2}{\mu Z N})$. Although these bounds vanish as $ZN \rightarrow \infty$, our analysis provides deeper understanding into the dependence on initialization, stochastic gradient noise variance and data heterogeneity as mentioned in Section \ref{sec:5c}. Last but not the least,
\cite{sun2021stability,deng2023stability,bars2023improved,ye2024ge,ye2024generalization} all overlook the influence of initiali- zation, also since they adopt the stringent bounded stochastic gradient assumption, even at the initial models $\vx^0_n$.

\setlength{\tabcolsep}{4pt}
\begin{table}
\caption{Comparisons with prior attack-free DSGD analysis in terms of assumptions and generalization guarantee with $\mu$-strongly convex loss. SM, BSG and HM are short for $L$-smoothness, bounded stochastic gradients and homogeneous local data distributions, respectively. }
\label{tab-com}
\begin{tabular}{ccccc}
\hline
   & SM & BSG & HM & Generalization Guarantee \\
  \hline
   \cite{sun2021stability,deng2023stability} & \ding{51} &\ding{51} &\ding{51} & $O(\frac{M^2}{\mu Z N})+O(\frac{1}{\mu})$\\
   \cite{bars2023improved,ye2024generalization} & \ding{51} &\ding{51} &\ding{55} & $O(\frac{M^2}{\mu Z N})$\\
   Ours  & \ding{51} &\ding{55} &\ding{55} &  $O(\frac{\|\bar\vx^0 - \vx_{\cS}^*\|^2}{\mu N Z}) + \tilde O(\frac{\sigma^2}{\mu N Z}) + O(\frac{\delta^2}{\mu N Z})$ \\
  \hline
\end{tabular}
\end{table}

\begin{remark}
Under the bounded stochastic gradient assumption,
the works of \cite{bars2023improved,ye2024ge,ye2024generalization} have established the generalization error bounds that are independent on $\lambda$. The reason is that such an assumption treats the stochastic gradient norm as a constant, thereby ignoring its dynamics during training that is dependent on $\lambda$. By removing the bounded stochastic gradient assumption in our analysis, $\lambda$ affects both the on-average stability and the term $\E_{\cS,\mathcal{L}} F_{\cS}(\mathcal{L}(\cS))$ in Lemma \ref{lemma-sta}, and thus impacts the generalization error. In particular, the constant $C_1 = O(\frac{1-\lambda}{\lambda^3})$ in Theorem \ref{the-oe} indicates that a larger $\lambda$, which corresponds to a denser network topology, leads to a smaller $\E_{\cS,\mathcal{L}} F_{\cS}(\mathcal{L}(\cS))$ and thus a smaller generalization error.
\end{remark}

\begin{remark}
    Establishing the generalization error bound of Algorithm \ref{DSGD} with the heterogeneous data and without the bounded stochastic gradient assumption is non-trivial. The main challenges lie in: (i) bounding $\E_{\cS,\mathcal{L}}\| \nabla F_{\cS}(\bar \vx^k) \|^2$ in \eqref{lemma}; (ii) establishing the on-average stability.

    For the first challenge, our analysis demonstrates that the boundedness of $\E_{\cS,\mathcal{L}}\|   \nabla F_{\cS}(\bar \vx^k) \|^2$ can be ensured by setting an appropriate step size  $\alpha^k$, which guarantees that the iterate $\bar\vx^k$  is always bounded and so as the squared gradient norm.

    To address the second challenge, a key step is to bound the impact of different training samples on the stochastic gradients. In addition to setting an appropriate step size, we also elaborate on characterizing the consensus error across the agents' local models. Thanks to the inequality $\E_{\cS,\cS^{(n,z)},\mathcal{L}}$ $\| \nabla f(\vx^{k}_n; \xi_{n,z}) - \nabla f(\vx^{k}_n; \xi'_{n,z})\| \leq$  $2 \E_{\cS,\mathcal{L}}$ $ \| \nabla f(\vx^{k}_n; \xi_{n,z})\|$,
    we are able to decompose $\E_{\cS,\mathcal{L}}  \frac{1}{N}\sum_{n \in \mathcal{N}}$ $\| \nabla f(\vx^{k}_n; \xi_{n,z})\|^2$
    into $  4 \sigma^2 \!+\! 4\delta^2 \!+ 4\E_{\cS,\mathcal{L}}  \| \nabla F_{\cS}(\bar\vx^{k})\|^2 \!+\! 4\E_{\cS,\mathcal{L}} \frac{1}{N}\sum_{n \in \mathcal{N}} \|  \nabla F_{\cS}(\vx^{k}_n) \! -\! \nabla F_{\cS}(\bar\vx^{k})\|^2  $. Further using Assumption \ref{assumption:Lip}, we can then bound the last summand with the consensus error  $\frac{L^2}{N} \sum_{n \in \mathcal{N}} \| \vx^{k}_n - \bar\vx^{k}\|^2 $, which is one major measure to characterize the dynamics of attack-free DSGD.

    While these two challenges can be straightforwardly addre- ssed by assuming bounded stochastic gradients as shown in \cite{sun2021stability,deng2023stability,bars2023improved,ye2024ge,ye2024generalization}, our analysis develops novel theoretical techniques, enabling the establishment of the generalization error in attack-free decentralized learning with heterogeneous data under relaxed assumptions.

    Our generalization error analysis of attack-free decentralized DSGD is also much more challenging than that of single-agent SGD. Specifically, it requires first examining the sperate impact of different training samples on each local model, followed by averaging across all agents. This process is further complicated by inter-agent communication, data heterogeneity and model disagreement among the agents.
\end{remark}

\section{Generalization Error Analysis for Byzantine-resilient DSGD}
\label{sec-ge-b}
Below, we proceed to analyze the on-average stability, optimization error and generalization error of Byzantine-resilient DSGD with heterogeneous data.
We emphasize that, in the Byzantine-resilient setup, $\N$ refers to the set of non-Byzantine agents and $\mathcal{B}$ refers to the set of Byzantine agents.

%
%

\subsection{On-average Stability of Byzantine-resilient DSGD}
In the presence of Byzantine agents, below we establish the on-average stability bound of Algorithm \ref{robust-DSGD}.

\begin{theorem}[On-average Stability of Byzantine-resilient DSGD with Heterogeneous Data]
\label{the-onsta-b}
Suppose that the robust aggregation rules $\{\A_n\}_{n\in \N}$ in Algorithm \ref{robust-DSGD} satisfy Definition \ref{definition:mixing-matrix},  the associated contraction constant satisfies $\rho < \rho^* := \frac{\lambda}{8\sqrt{N}}$ and
Assumptions \ref{assumption:convex}--\ref{assumption:heterogeneity} hold for all non-Byzantine agents $n \in \N$. If we set a proper step size $\alpha^k= \frac{2}{\mu(k+k_1)} $, where $k_1$ is sufficiently large, then at any given time $k$,
if the virtual mixing $W$ is only row stochastic, the on-average stability of Algorithm \ref{robust-DSGD} is bounded by
\begin{align}
        \label{thm-on-b}
        & \frac{1}{NZ} \sum_{n \in \mathcal{N}}\sum_{z=1}^{Z} \E_{\cS,\cS^{(n,z)},\mathcal{L}} \left[ \|\cL(\cS) - \cL(\cS^{(n,z)})\|^2 \right]  \\
    \leq
   &  \frac{C_2 \sum_{k'=0}^{k-1} \sum_{n \in \mathcal{N}}\sum_{z=1}^{Z} \E  \| \nabla f(\vx^{k'}_n; \xi_{n,z}) \|^2 / (k'+k_1) }{ N^3 Z^2 (k+k_1-1)}  \nonumber \\
    & + \frac{C_3 \sum_{k'=0}^{k-1} \sum_{n \in \mathcal{N}}\sum_{z=1}^{Z} \E  \| \nabla f(\vx^{k'}_n; \xi_{n,z}) \|^2  }{ N^3 Z^2  (k+k_1-1)}  \nonumber \\
    & + \frac{C_4 (\sigma^2 + \delta^2) }{ k+k_1-1}
        +  \frac{C_5  (4\rho^2 N+\chi^2) (\sigma^2 + \delta^2) }{ k+k_1-1}  \nonumber \\
     &+ C_6 (4\rho^2 N+\chi^2) (\sigma^2 + \delta^2); \nonumber
\end{align}
if the virtual mixing $W$ is doubly stochastic, the on-average stability of Algorithm \ref{robust-DSGD} is bounded by

\begin{align}
        \label{thm-on-b-1}
        & \frac{1}{NZ} \sum_{n \in \mathcal{N}}\sum_{z=1}^{Z} \E_{\cS,\cS^{(n,z)},\mathcal{L}} \left[ \|\cL(\cS) - \cL(\cS^{(n,z)})\|^2 \right]  \\
    \leq
   &  \frac{16 \sum_{k'=0}^{k-1} \sum_{n \in \mathcal{N}}\sum_{z=1}^{Z} \E  \| \nabla f(\vx^{k'}_n; \xi_{n,z}) \|^2 / (k'+k_1) }{ N^2 Z^2 (k+k_1-1)}  \nonumber \\
    & + \frac{8 \sum_{k'=0}^{k-1} \sum_{n \in \mathcal{N}}\sum_{z=1}^{Z} \E  \| \nabla f(\vx^{k'}_n; \xi_{n,z}) \|^2  }{ N^2 Z^3  (k+k_1-1)}  \nonumber \\
     &   +  \frac{2 C_5 \rho^2 N (\sigma^2 + \delta^2) }{ k+k_1-1}
      + 2 C_6 \rho^2 N(\sigma^2 + \delta^2). \nonumber
\end{align}
Here, $C_2,C_3,C_4,C_5,C_6>0$ are constants.
\end{theorem}


Under the decaying step size $\alpha^k = \frac{2}{\mu(k + k_1)}$, in which $k_1$ is introduced to prevent the initial step size from being too large,
Theorem \ref{the-onsta-b} establishes the on-average stability bounds for By- zantine-resilient DSGD with both row and doubly stochastic virtual mixing matrices.
In the presence of Byzantine agents, the contraction constant $\rho$ is generally nonzero, and the virtual mixing matrix is typically only row stochastic. In this case, the third, fourth and fifth terms at the right-hand side of \eqref{thm-on-b} arise. Among them, the third and fourth terms depend on the time $k$. The fifth term is constant and independent on both the sample size $Z$ and the time $k$, consistent with the observations in \cite{ye2024ge,ye2024generalization}.
If the robust aggregation rules are well-designed such that the virtual mixing matrix is doubly stochastic (i.e., $\chi^2=0$), then the third and fourth terms at the right-hand side of \eqref{thm-on-b-1} appear instead. In the absence of Byzantine agents (i.e., $\rho=0$), both of these terms in \eqref{thm-on-b-1} vanish, leading the bound to recover the attack-free one in Theorem \ref{the-onsta}.

\begin{remark}
For establishing the on-average stability bound of Byzantine-resilient DSGD, we must address two additional challenges compared to that of attack-free DSGD: (i) Bounding the approximation error introduced by those Byzantine agents; (ii) Analyzing the iterates of $ \|\bar\vx^{k+\frac{1}{2}} - \bar\vx'^{k+\frac{1}{2}}_{(n,z)} \|$ when the virtual mixing matrix is only row stochastic. 
Here, $\bar\vx^{k+\frac{1}{2}}_{(n,z)} = \frac{1}{N}\sum_{m \in \N} \vx'^{k+\frac{1}{2}}_m$, where $\vx'^{k+\frac{1}{2}}_m$ is the local model of agent $m$ at time $k+\frac{1}{2}$ when trained on the perturbed global dataset $\cS^{(n,z)}$.


For the first challenge, in the presence of Byzantine agents $(\rho \neq 0)$, an approximation error arises due to the discrepancy $\bar \vx^{k+1} \neq \bar \vx^{k+\frac{1}{2}}$ compared to attack-free DSGD.
To bound the approximation error of $\|\bar\vx^{k+1} - \bar\vx^{k+\frac{1}{2}} \|^2$,
we decompose it utilizing $ \|\bar\vx^{k+1} - \bar\vx^{k+\frac{1}{2}} \|^2  \leq 2\| \frac{1}{N} \sum_{n \in \N}(\vx_n^{k+1} - \hat\vx_n^{k+\frac{1}{2}}) \|^2
+ 2\| \frac{1}{N} \sum_{n \in \N} (\hat\vx_n^{k+\frac{1}{2}} - \vx_n^{k+\frac{1}{2}})\|^2$. For the first term at the right-hand side, we adopt the contraction constant $\rho$ from \eqref{inequality:robustness-of-aggregation-local} in Definition \ref{definition:mixing-matrix} to bound the increment. For the second term, if the virtual mixing matrix is doubly stochastic, this term vanishes; otherwise, we utilize the constant $\chi^2$ to quantify the deviation introduced by the non-doubly stochastic nature of the inves- tigated mixing process.

For the second challenge, when the virtual mixing matrix is only row stochastic,
we are unable to have $\frac{1}{N} \sum_{n \in \N} \|\hat\vx_n^{k+\frac{1}{2}}$ $- \hat\vx_n'^{k+\frac{1}{2}} \|^2 \leq   \frac{1}{N} \sum_{n \in \mathcal{N}} \sum_{m \in \mathcal{N}} w_{nm} \| \vx_m^{k+\frac{1}{2}} - \vx_m'^{k+\frac{1}{2}}  \|^2 \leq \frac{1}{N} $ $\sum_{n \in \mathcal{N}} \|\vx_n^{k+\frac{1}{2}} - \vx_n'^{k+\frac{1}{2}} \|^2 $
using $\sum_{n \in \mathcal{N}}  w_{nm} =1 $. As a result,
we cannot analyze the iterates of $\|\vx_n^{k+\frac{1}{2}} - \vx_n'^{k+\frac{1}{2}} \|^2$  in a manner similar to the single-agent case, which would otherwise lead to a tighter on-average stability bound.
Instead, we must directly analyze the iterates of the average model $ \|\bar\vx^{k+\frac{1}{2}} - \bar\vx'^{k+\frac{1}{2}}_{(n,z)} \|$ and utilize the consensus error to bound the additional deviation term $\frac{1}{N}\sum_{n \in \N}\|\nabla f(\bar \vx^k,\xi)-\nabla f(\vx_n^k,\xi)\|^2$.
\end{remark}

\subsection{Optimization Error of Byzantine-resilient DSGD}

The optimization error of Byzantine-resilient DSGD with heterogeneous data has also been studied. For completeness, we provide an upper bound for the optimization error in the following theorem.

\begin{theorem}[Optimization Error of Byzantine-resilient DSGD with Heterogeneous Data \cite{wu2022byzantine,ye2024tradeoff}]
\label{the-oe-b}
Suppose that the robust aggregation rules $\{\A_n\}_{n\in \N}$ in Algorithm \ref{robust-DSGD} satisfy Definition \ref{definition:mixing-matrix},  the associated contraction constant satisfies $\rho < \rho^* := \frac{\lambda}{8\sqrt{N}}$ and
Assumptions \ref{assumption:convex}--\ref{assumption:heterogeneity} hold for all non-Byzantine agents $n \in \N$. If we set a proper step size $\alpha^k= \frac{2}{\mu(k+k_1)} $, where $k_1$ is sufficiently large, then at any given time $k$, the optimization error of Algorithm \ref{robust-DSGD} is bounded by
\begin{align}
        \label{thm-oe-b}
        &   \E_{\cS,\mathcal{L}} [F_{\cS}(\bar\vx^{k}) - F_{\cS} (\vx_{\cS}^*)]  \leq   \frac{C_7 (\sigma^2 + \delta^2)}{ k+k_1-1 } \\
  &  +  \frac{L(k_1-1)  \|\bar \vx^0 - \vx_{\cS}^* \|^2}{2(k+k_1-1)} + \frac{4L\sigma^2 ln(k+k_1-1)}{ \mu^2 N (k+k_1-1)} \nonumber \\
     & + \frac{C_8 (4\rho^2 N+\chi^2) (\sigma^2 + \delta^2)}{k+k_1-1} + C_9  (4\rho^2 N+\chi^2)   (\sigma^2 + \delta^2).  \nonumber
\end{align}
Here, $C_7,C_8,C_9>0$ are constants.
\end{theorem}

Theorem \ref{the-oe-b} establishes the optimization error bound for Byzantine-resilient DSGD. When no Byzantine agents are present (i.e., $\rho=0$) and the virtual mixing matrix is doubly stochastic (i.e., $\chi=0$), the optimization error bound in \eqref{thm-oe-b} reduces to its attack-free counterpart in Theorem \ref{the-oe}.
Compared to the optimization error bound established for Algorithm \ref{DSGD} in Theorem \ref{the-oe}, Algorithm \ref{robust-DSGD} can only converge to a neighborhood of the optimal solution $\vx_{\cS}^*$ in the presence of Byzantine agents. As indicated by the fifth term at the right-hand side of \eqref{thm-oe-b}, the size of this neighborhood is influenced by the robust agg- regation rules (characterized by $\rho$ and $\chi$), stochastic gradient noise and data heterogeneity. Moreover, the convergence rate remains $\tilde O(\frac{1}{k})$, matching that of its attack-free counterpart in Theorem \ref{the-oe}.

\subsection{Generalization Error of Byzantine-resilient DSGD}

\begin{theorem}[Generalization Error of DSGD with Heterogeneous Data]
\label{the-ge-b}
Let $k \asymp \mu N Z$. Suppose that the robust aggregation rules $\{\A_n\}_{n\in \N}$ in Algorithm \ref{robust-DSGD} satisfy Definition \ref{definition:mixing-matrix},  the associated contraction constant satisfies $\rho < \rho^* := \frac{\lambda}{8\sqrt{N}}$ and
Assumptions \ref{assumption:convex}--\ref{assumption:heterogeneity} hold for all non-Byzantine agents $n \in \N$. If we set a proper step size $\alpha^k= \frac{2}{\mu(k+k_1)}$, where $k_1$ is sufficiently large, the generalization error of Algorithm \ref{robust-DSGD} is bounded by

\begin{align}
        \label{thm-ge-1-b}
        & \E_{\cS,\mathcal{L}}[F(\bar \vx^k)-F_{\cS}(\bar \vx^k))] \\
        \leq &  O(\frac{\|\bar\vx^0 - \vx_{\cS}^*\|^2}{\mu N Z})+  \tilde O(\frac{\sigma^2}{\mu N Z}) + O(\frac{\delta^2}{\mu N Z}) \nonumber
        \\ &+ O( (4\rho^2 N + \chi^2 ) (\sigma^2+\delta^2)) \nonumber.
\end{align}

\end{theorem}

Theorem \ref{the-ge-b} establishes the generalization error bound for Byzantine-resilient DSGD. When no Byzantine agents are present (i.e., $\rho=0$) and the virtual mixing matrix is doubly stochastic (i.e., $\chi=0$), the generalization error bound in \eqref{thm-ge-1-b} reduces to its attack-free counterpart in Theorem \ref{the-ge}.
Compared to the generalization error bounds established for Algorithm \ref{DSGD} in Theorem \ref{the-ge}, the presence of Byzantine agents introduces an additional error term of order $O( (4\rho^2 N + \chi^2 ) (\sigma^2+\delta^2)) $, which is independent on the sample size. Such an additional error term indicates that even as the sample size approaches infinity, the generalization error does not vanish. Instead, similar to the optimization error, it remains bounded by a constant that depends on the robust aggregation rule, stochastic gradient noise and data heterogeneity.
The intuition is that the malicious updates sent by Byzantine agents can be regarded as corrupted samples within the training dataset. Consequently, the training dataset is unable to accurately approximate the underlying data distribution in the presence of Byzantine agents, even with infinite training samples.

The work of \cite{ye2024generalization} has established the first generalization error bound for Byzantine-resilient DSGD. Nevertheless, the analysis must rely on the stringent assumption of bounded stochastic gradients and does not account for the impact of data heterogeneity on the generalization error.
In contrast, we remove this assumption and provide new insights into the interaction between the impact of Byzantine agents and data heterogeneity. Specifically, when the data distributions among non-Byzantine agents vary significantly, the negative impact of Byzantine agents on the generalization error becomes more severe -- an aspect not explored in \cite{ye2024generalization}.

\subsection{Cooperation Gain in the Presence of Byzantine Agents}

According to Theorem \ref{the-ge-b}, the presence of Byzantine agents introduces non-vanishing error terms in the generalization error.
This naturally raises the question:
\textit{From a generalization perspective, under what conditions is collaboration beneficial in the presence of Byzantine attacks?}

To answer this question, we must establish the generalization error bound of the algorithm when training independently on each agent's local dataset under data heterogeneity.
When the data distributions are heterogeneous among non-Byzantine agents, we leverage tools from domain adaptation theory \cite{blitzer2007learning,shui2022novel,allouahfine} to rigorously quantify the discrepancy between local data distributions $\cD_n$ and the global data distribution $\cD$.
Formally, for any $n \in \N$, the function that measures the discrepancy between two distributions is defined as
\begin{align}
    |F_{\cD_n}(\vx) - F_{\cD}(\vx)| \leq \Phi (\cD_n,\cD) , \vx \in \R^d
\end{align}

\begin{remark}
    The discrepancy measure function $\Phi$ has been introduced for various losses in machine learning.
    For the $0$-$1$ loss in binary classification, the discrepancy measure is defined as $\Phi (\cD_n,\cD) = 2 \sup_{h \in \mathcal{H}} | \mathbb{P}_{\mathcal{D}_n}(h(\vx) = 1) - \mathbb{P}_{\mathcal{D}}(h(\vx) = 1) | $, where $\mathcal{H}$ is the hypothesis class; see \cite{blitzer2007learning}.
    For entropy-based losses (e.g., the cross-entropy loss), the discrepancy measure is given by
    $\Phi (\cD_n,\cD)  = \frac{1}{2} D_{\text{KL}}(\mathcal{D}_n \!\parallel\! \mathcal{M}) + \frac{1}{2} D_{\text{KL}}(\mathcal{D} \!\parallel\! \mathcal{M})$,
    where $\mathcal{M} = \frac{1}{2} (\mathcal{D}_n + \mathcal{D})$ represents the mixture distri- bution and \( D_{\text{KL}} \) denotes the Kullback-Leibler (KL) divergence; see \cite{shui2022novel}. For more general losses, hypothesis space-dependent integral probability metrics can be used, as discussed in \cite{allouahfine}.

\end{remark}

Given the task defined in Section \ref{sec2} under Byzantine attacks, the goal is to ensure that the trained model performs well on the underlying global data distribution $\cD$.
If a non-Byzantine agent $n$ chooses not to cooperate, its training set is limited to $\cS_n$.  In this case, the generalization error is given by
$\E_{\cS_n,\mathcal{L}}[F(\cL(\cS_n))-F_{\cS_n}( \cL(\cS_n))]$.
We now establish the generalization error bound for running SGD independently on the dataset $\cS_n$ without cooperation, for any non-Byzantine agent $n \in \N$.

\begin{theorem}[Generalization Error of SGD on the Dataset $\cS_n$ without Cooperation]
\label{the-ge-s}
Let $k \asymp \mu N Z$. Under Assumptions \ref{assumption:convex}--\ref{assumption:heterogeneity}, if we set a proper step size $\alpha^k= \frac{1}{\mu(k+k_0)}$, where $k_0$ is sufficiently large, the generalization error of SGD on the dataset $\cS_n$ without cooperation is bounded by

\begin{align}
        \label{thm-ge-s}
        & \E_{\cS_n,\mathcal{L}}[F( \vx^k_n)-F_{\cS_n}( \vx^k_n))]  \\
        \leq &  O(\frac{\|\vx^0_n - \vx_{\cS_n}^*\|^2}{ \mu Z})+  \tilde O(\frac{\sigma^2}{ \mu Z}) + \Phi (\cD_n,\cD).  \nonumber
\end{align}

\end{theorem}

From Theorem \ref{the-ge-s}, we conclude that with data heterogeneity, training independently without cooperation also results in a non-vanishing additional error term in the generalization error bound. This term is independent on the sample size and depends solely on the discrepancy between the local and global data distributions. Notably, when the local and global distributions are identical, this term vanishes.
In contrast to Theorem \ref{the-ge}, which considers full cooperation without Byzantine attacks, Theorem \ref{the-ge-s} cannot leverage cooperation to improve the bound from $O(\frac{1}{Z})$ to $O(\frac{1}{NZ})$.
Consequently, in the absence of Byzan- tine agents, cooperation is always recommended, as it allows the non-Byzantine agents to utilize the other training samples to improve the bound by a factor of $\frac{1}{N}$ and mitigate the constant error induced by data heterogeneity. Instead, when the Byzantine agents are present, a comparison of Theorems \ref{the-ge-b} and \ref{the-ge-s} provides insights into when cooperation is beneficial from a generalization perspective.

\begin{itemize}[leftmargin=*]
\item \textbf{When the data distributions across non-Byzantine agents are identical, cooperation is not necessary.} In this case, $\Phi (\cD_n,\cD) =0$ and the generalization error of independent training diminishes to zero as the sample size $Z$ approaches infinity.
However, cooperation may introduce the negative impact of Byzantine agents, leading to non-vanishing error terms on the generalization error, regardless of the sample size.

\item \textbf{When the data distributions across non-Byzantine agents vary significantly and the local sample size is limited, cooperation is recommended.}
When the local samples are limited, cooperation with other non-Byzantine agents can be beneficial, improving certain terms from $O(\frac{1}{Z})$ to $O(\frac{1}{NZ})$.
In the presence of significant data heterogeneity, where $\Phi (\cD_n,\cD)$ is large, an appropriately designed aggregation rule can mitigate the impact of data heterogeneity, reducing the generalization error to $O( (4\rho^2 N + \chi^2 )$ $(\sigma^2+\delta^2))$.
\end{itemize}


\section{Numerical Experiments}
\label{sec-num}

In the numerical experiments, we generate an Erdos-R{\'e}nyi network topology of 10 agents, in which each pair of agents is connected with a probability of 0.7.
In the attack-free setup, all 10 agents are non-Byzantine. In the Byzantine-resilient setup, 2 out of 10 agents are selected as Byzantine. We investigate the squared $\ell_2$-norm regularized softmax regression task with strongly convex losses on the MNIST dataset, which consists of 60,000 training samples and 10,000 testing samples across 10 classes.
We test various levels of data heterogeneity among the agents' local data distributions to assess its impact on the generalization error. Specifically, the data heterogeneity is con- trolled by assigning the training samples to all agents accord- ing to the Dirichlet distribution $Dir(\beta)$,
which is commonly used to distribute heterogeneous data among the agents within federated and decentralized learning \cite{li2020federated}.
Therein, $\beta$ is the concentration parameter; a smaller $\beta$ results in more
heteroge- neous data partition, whereas a larger $\beta$ leads to more balanced partition. We implement attack-free DSGD outlined in Algorithm \ref{DSGD} and Byzantine-resilient DSGD outlined in Algorithm \ref{robust-DSGD} equipped with different robust aggregation rules: TM \cite{fang2022bridge}, IOS \cite{wu2022byzantine} and SCC \cite{he2022byzantine}.

During the training process, the step size is chosen as  $\alpha^k = \frac{1}{0.01k+1}$. The generalization error of the average model $\bar \vx^k$ is approximated by the difference between the losses calculated on the training samples and unseen testing samples, following \cite{deng2023stability,zhu2022topology,bars2023improved,ye2024ge,ye2024generalization}.

When the Byzantine agents are present, we consider the fo- llowing Byzantine attacks.

\noindent\textbf{Gaussian Attack\cite{schroth2024sensitivity}.} The Byzantine agents transmit messages with elements drawn from a Gaussian distribution with mean $0$ and variance $1$. \\
\noindent\textbf{Sample-Duplicating Attack\cite{ye2024generalization}.} The Byzantine agents collaboratively choose one non-Byzantine agent and consistently replicate its messages when communicating. This behavior is equivalent to the Byzantine agents duplicating the training samples of the selected non-Byzantine agent. \\
\noindent\textbf{A-Little-Is-Enough (ALIE) Attack\cite{baruch2019little}.} To non-Byzantine agent $n$, its Byzantine neighbors send $\frac{1}{|\N_n|}\sum_{m\in \N_n }\vx_{m,m}^{k+\frac{1}{2}} + r_n^k \Delta_n^k$ at time $k$, in which $\Delta_n^k$ is the coordinate-wise standard deviation of $\{\vx_{m,m}^{k+\frac{1}{2}}\}_{m\in \N_n}$ and $r_n^k$ is the scale factor. \\
\noindent\textbf{Sign-Flipping Attack\cite{wu2022byzantine}.} At time $k$, each Byzantine neighbor of non-Byzantine agent $n$ multiplies $\frac{1}{|\N_n|}\sum_{m\in \N_n }\vx_{m,m}^{k+\frac{1}{2}}$ with a negative constant $-1$ and sends the result.


\subsection{Experimental Results for Attack-free DSGD}





\begin{figure}[t]
\centering
\begin{minipage}{0.49\linewidth}
    \centering
    \includegraphics[width=0.9\linewidth]{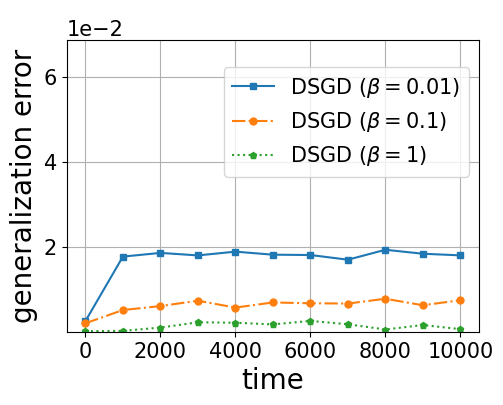}
    \caption{Generalization error of DSGD under different levels of data heterogeneity.}
    \label{gene-noniid}
\end{minipage}
\hfill
\begin{minipage}{0.49\linewidth}
    \centering
    \includegraphics[width=0.925\linewidth]{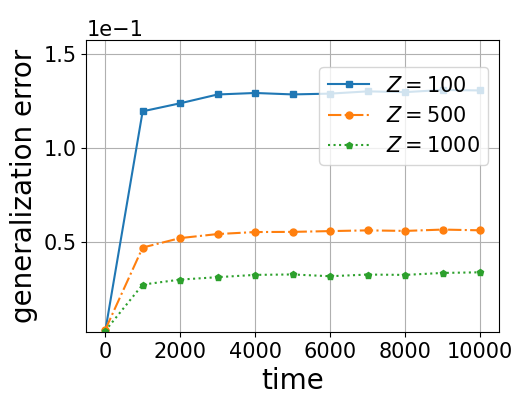}
    \caption{Generalization error of DSGD across different training sample sizes per agent when $\beta=0.01$.}
    \label{gene-size}
\end{minipage}
\end{figure}

Fig. \ref{gene-noniid} illustrates the generalization errors of Algorithm \ref{DSGD} under various levels of data heterogeneity.
We can observe that a higher level of data heterogeneity (a smaller concentration parameter $\beta$) is often associated with a larger generalization error, which corroborates the theoretical finding presented in Theorem \ref{the-ge}. Fig. \ref{gene-size} further depicts the generalization error of Algorithm \ref{DSGD} for different training sample sizes per agent when $\beta=0.01$.
The numerical results show that the generalization error decreases when the training sample size increases.
This phenomenon also supports our theoretical finding in Theorem \ref{the-ge}, which states that the generalization error is inversely proportional to the sample size.

\subsection{Experimental Results for Byzantine-resilient DSGD}




Fig. \ref{fig:br-dsgd} illustrates the generalization errors of Algorithm \ref{robust-DSGD} with different robust aggregation rules $\A_n$, including TM, IOS and SCC, under varying levels of data heterogeneity. Consistent with the observations in Fig. \ref{gene-noniid}, we observe that as data heterogeneity increases, the generalization errors of Algorithm \ref{robust-DSGD} with all robust aggregation rules also increase. This aligns with the theoretical results in Theorem \ref{the-ge-b}.
Furthermore, Fig. \ref{fig:robust} compares the generalization errors of attack-free DSGD and Byzantine-resilient DSGD when the concentration parameter is set to $\beta = 0.01$. It is evident that the generalization error of Byzantine-resilient DSGD is significantly larger than that of attack-free DSGD, confirming that the presence of Byzantine agents introduces additional error terms in the generalization error, as established in Theorem \ref{the-ge-b}.  By jointly analyzing Fig. \ref{gene-noniid} and Fig. \ref{fig:robust}, we observe that as data heterogeneity increases, the additional error induced by Byzantine agents becomes more severe, further validating the results in Theorem \ref{the-ge-b}.

\begin{figure*}[htbp]
    \centering
    \includegraphics[width=0.74\linewidth]{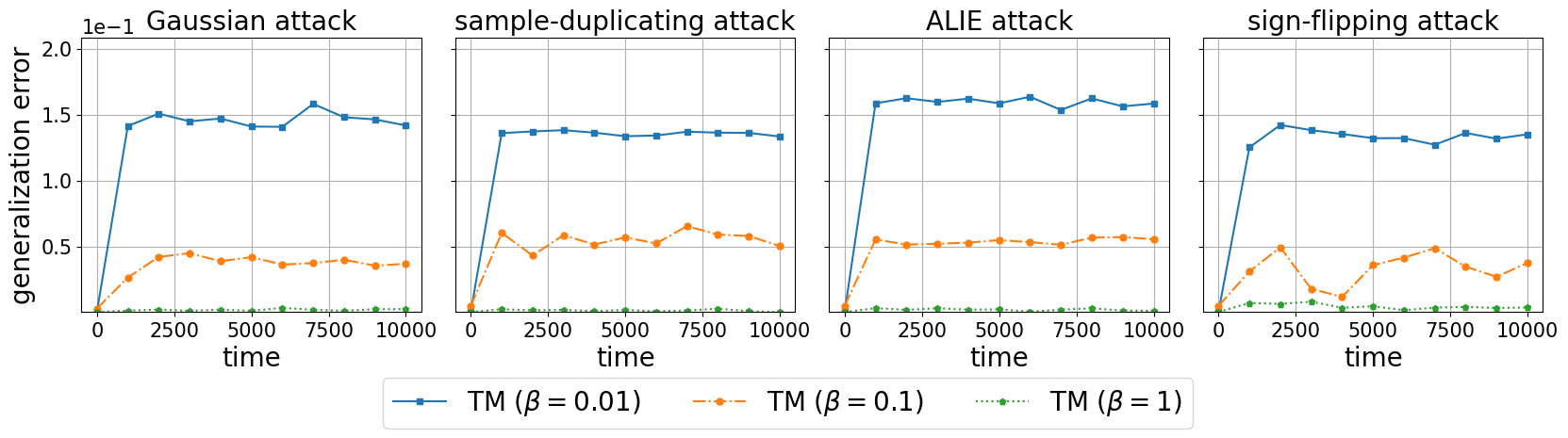}\\
    \vspace{4mm}
    \includegraphics[width=0.74\linewidth]{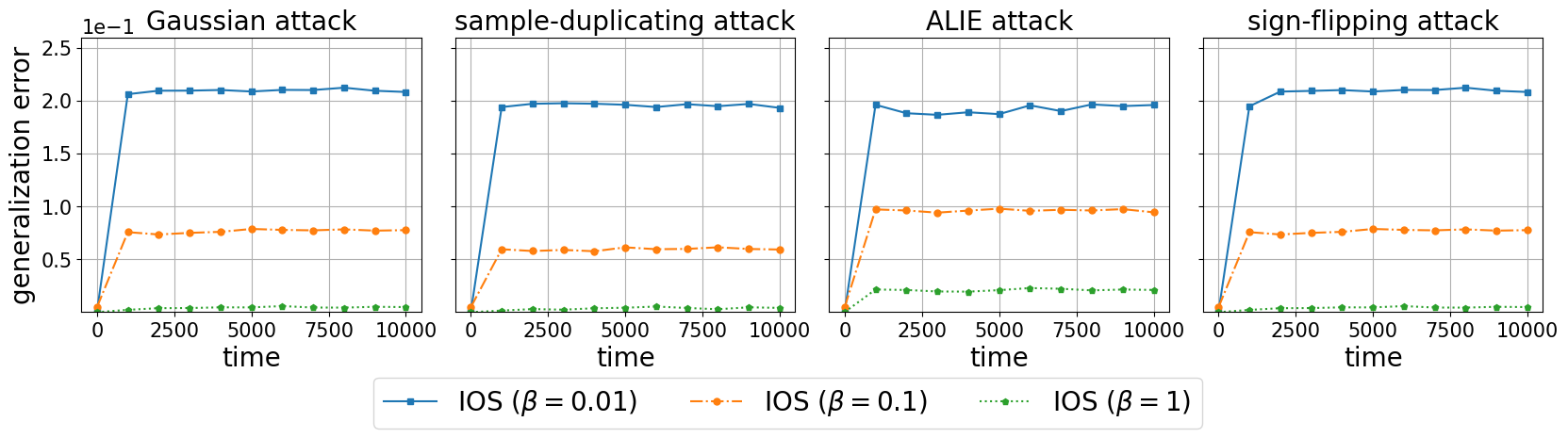} \\
     \vspace{4mm}
    \includegraphics[width=0.74\linewidth]{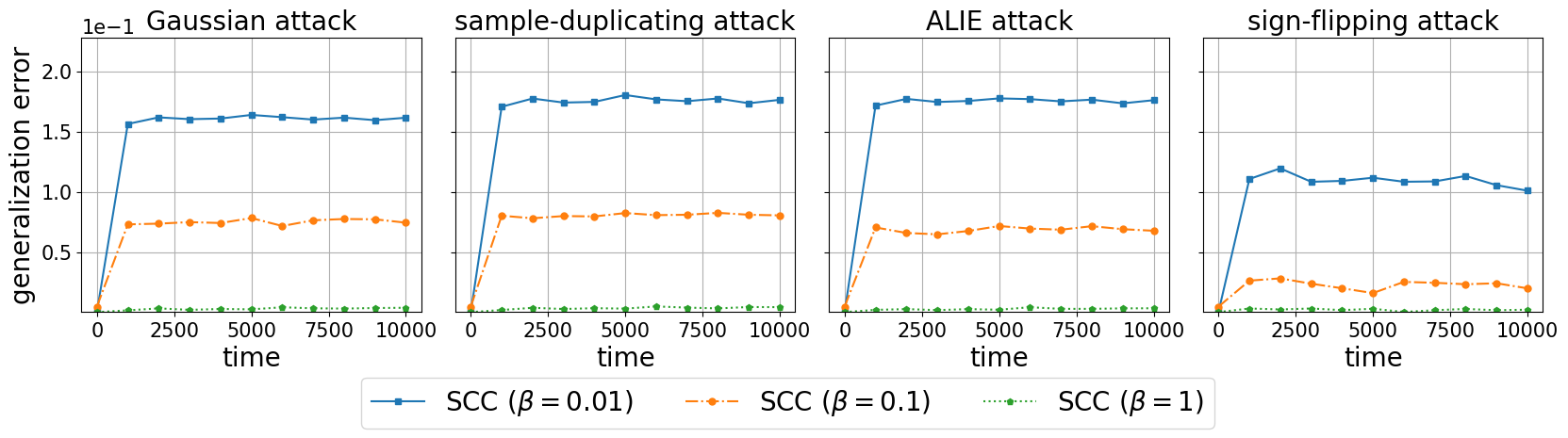}

    \caption{Generalization error of Byzantine-resilient DSGD under different levels of data heterogeneity with various aggregation rules.}
    \label{fig:br-dsgd}
\end{figure*}

\begin{figure*}[htbp]
    \centering
    \includegraphics[width=0.74\linewidth]{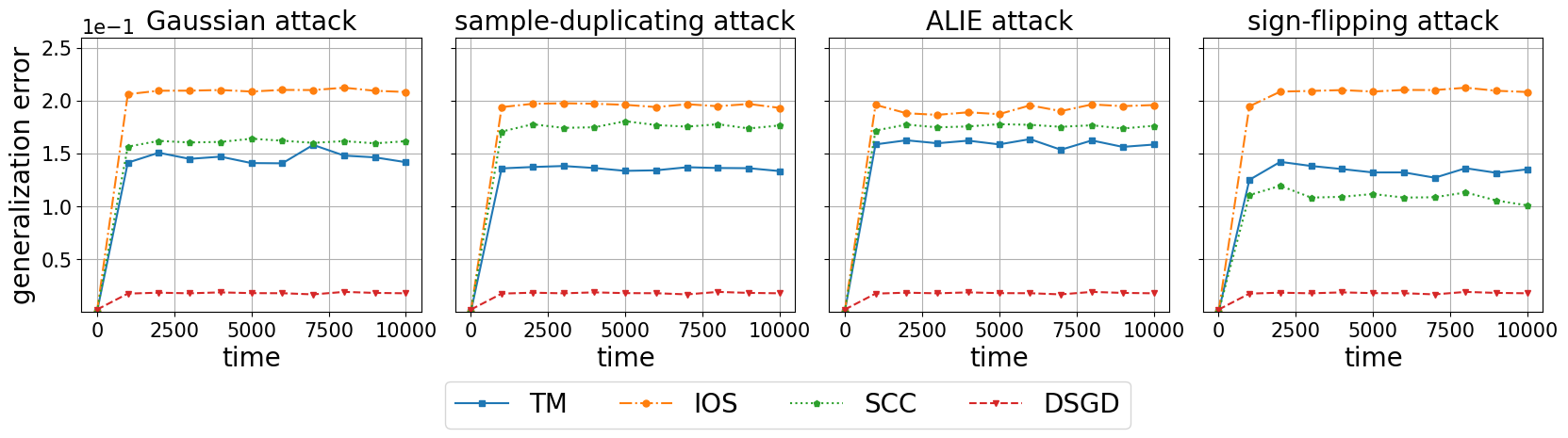}
    \caption{Generalization error of attack-free DSGD and Byzantine-resilient DSGD when $\beta=0.01$.}
    \label{fig:robust}
\end{figure*}

\subsection{Cooperation Gain of Byzantine-resilient DSGD}

To highlight the cooperation gain in the presence of Byzantine agents, we set the concentration parameter to $\beta = 0.01$ and randomly select one non-Byzantine agent to train independently on its own training samples without cooperation. The generalization errors of single-agent SGD without cooperation and Byzantine-resilient DSGD with cooperation are illustrated in Fig. \ref{fig:co}. We observe that the generalization error of single-agent SGD is significantly larger than that of Byzantine-resi- lient DGSD when $\beta = 0.01$. This is consistent with our theo- retical results in Theorem \ref{the-ge-s}. In this case, no cooperation avoids the negative impact of Byzantine agents; nevertheless, it also prevents the benefits from the training samples of other non-Byzantine agents and suffers from the discrepancies between local and global distributions. Therefore, when data heterogeneity is high, cooperation is both necessary and beneficial from a generalization perspective.

\begin{figure*}[htbp]
    \centering
    \includegraphics[width=0.74\linewidth]{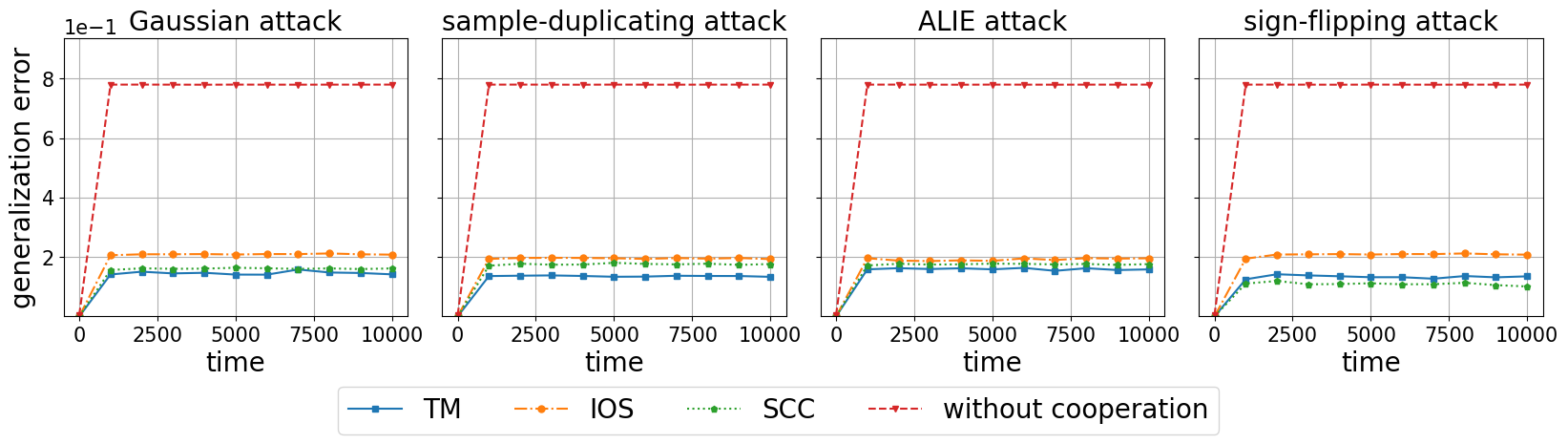}
    \caption{Generalization error of SGD without cooperation and Byzantine-resilient DSGD when $\beta=0.01$.}
    \label{fig:co}
\end{figure*}


\section{Conclusions}
\label{sec-con}

In this paper, we present the generalization guarantees for attack-free and Byzantine-resilient decentralized learning with heterogeneous data. Our generalization error analysis removes the stringent assumption of bounded stochastic gradients and reveals the impact of data heterogeneity in decentralized learn- ing. We also reveal the negative effect of Byzantine attacks on the generalization error under data heterogeneity and analyze the conditions under which cooperation remains beneficial in the presence of Byzantine agents from a generalization perspec- tive. The novel analytical techniques that we develop to estab- lish the generalization error bounds have potentials to inspire future works in this research topic. Numerical experiments on strongly convex and non-convex tasks are conducted, and the results validate our theoretical findings.

\bibliographystyle{IEEEtran}
\bibliography{main}

\begin{appendices}

\section{Proof of Theorem \ref{the-onsta}}
\begin{proof}
Following \cite{sun2021stability,deng2023stability,zhu2022topology,bars2023improved,ye2024ge,ye2024generalization}, we denote the output of attack-free DGSD on the datasets $\cS$ and $\cS^{(i,z)}$ at time $k$ as  $\cL(\cS) = \bar\vx^k = \frac{1}{N}\sum_{n \in \N} \vx_n^k$ and $\cL(\cS^{(i,z)})=\bar\vx'^k_{(i,z)} = \frac{1}{N}\sum_{n \in \N} \vx'^k_n$, respectively.
Here, $\vx'^k_n$ denotes the local model of agent $n$ at time $k$ when trained on the perturbed global dataset $\cS^{(i,z)} = \{\cS_{m} | m \in \mathcal{N},  m \neq i\} \cup \cS_i^{(z)}$, and $\bar\vx'^k_{(i,z)}$ is the average of $\vx'^k_n$.
For the distance between $\bar\vx^{k+1}$ and $\bar\vx'^{k+1}_{(i,z)}$, we have
\begin{align}
\label{g-5}
    &\|\bar\vx'^{k+1}_{(i,z)} - \bar\vx^{k+1} \|^2 \\
    \leq & \frac{1}{N}\sum_{n \in \N} \|  \vx'^{k+1}_{n} - \vx^{k+1}_n  \|^2 \nonumber \\
     \leq &  \frac{1}{N} \sum_{n \in \N} \sum_{m \in \N} w_{nm} \|     \vx_m^{k+\frac{1}{2}} - \vx_m'^{k+\frac{1}{2}}  \|^2 \nonumber \\
     \leq& \frac{1}{N} \sum_{n \in \N} \|\vx_n^{k+\frac{1}{2}} - \vx_n'^{k+\frac{1}{2}} \|^2.  \nonumber
\end{align}

For agent $i$ and at time $k$,
the probability of  Algorithm \ref{DSGD} selecting the same training sample from  $\cS_i$ and $\cS_i^{(z)}$  is $1-\frac{1}{Z}$. In this case,
according to Lemma 6 in \cite{sun2021stability}, when the loss is strongly convex and Assumption \ref{assumption:Lip} holds, choosing $\alpha^k \leq \frac{1}{L}$ yields
\begin{align}
\label{g-6}
    & \frac{1}{N} \sum_{n \in \N} \|\vx_n^{k+\frac{1}{2}} - \vx_n'^{k+\frac{1}{2}} \|^2 \\
    \leq & (1-\alpha^k \mu)^2 \frac{1}{N} \sum_{n \in \N} \|\vx_n^{k} - \vx_n'^{k} \|^2. \nonumber
\end{align}


For agent $i$ and at time $k$,
the probability of
 Algorithm \ref{DSGD} selecting the different training samples from  $\cS_i$ and $\cS_i^{(z)}$
is $\frac{1}{Z}$.
In this case, by Lemma 6 in \cite{sun2021stability}, when the loss is strongly convex and Assumption \ref{assumption:Lip} holds, for any $p>0$, choosing $\alpha^k \leq \frac{1}{L}$ yields
\begin{align}
    \label{g-7}
    & \frac{1}{N} \sum_{n \in \N} \|\vx_n^{k+\frac{1}{2}} - \vx_n'^{k+\frac{1}{2}} \|^2 \\
   = & \frac{1}{N} \| \vx_i^k -\alpha^k \nabla f(\vx^{k}_i; \xi_{i,z})-\vx'^k_i + \alpha^k \nabla f(\vx'^{k}_i; \xi'_{i,z}) \|^2  \nonumber \\
   & +    \frac{1}{N}  \sum_{n \in \N /\{i\} } \|\vx_n^{k+\frac{1}{2}} - \vx_n'^{k+\frac{1}{2}} \|^2 \nonumber \\
   \leq & \frac{1+p}{N} \| \vx_i^k -\alpha^k \nabla f(\vx^{k}_i; \xi'_{i,z})-\vx'^k_i + \alpha^k \nabla f(\vx'^{k}_i; \xi'_{i,z}) \|^2  \nonumber \\
   & + \frac{(\alpha^k)^2(1+p^{-1})}{N}  \| \nabla f(\vx^{k}_i; \xi_{i,z}) - \nabla f(\vx^{k}_i; \xi'_{i,z})\|^2  \nonumber \\
   & +    \frac{1}{N}  \sum_{n \in \N /\{i\} } \|\vx_n^{k+\frac{1}{2}} - \vx_n'^{k+\frac{1}{2}} \|^2 \nonumber \\
   \leq &  \frac{(1-\alpha^k \mu)^2 (1+p)}{N}  \sum_{n \in \N} \|\vx_n^{k} - \vx_n'^{k} \|^2  \nonumber \\
    & +  \frac{(\alpha^k)^2(1+p^{-1})}{N}  \| \nabla f(\vx^{k}_i; \xi_{i,z}) - \nabla f(\vx^{k}_i; \xi'_{i,z})\|^2.   \nonumber
\end{align}

Combining the two cases in \eqref{g-6} and \eqref{g-7}, we obtain

\begin{align}
    \label{g-10}
    & \frac{1}{N} \sum_{n \in \N}   \E \|\vx_n^{k+1} - \vx_n'^{k+1} \|^2 \\
    \leq &   (1-\alpha^k \mu)^2 (1+\frac{p}{Z})  \frac{1}{N}  \sum_{n \in \N} \E  \|\vx_n^{k} - \vx_n'^{k} \|^2 \nonumber \\
     & +  \frac{ 4 (\alpha^k)^2 (1+p^{-1}) }{NZ}  \E \| \nabla f(\vx^{k}_i; \xi_{i,z}) \|^2.  \nonumber
\end{align}

If we set  $p = Z \alpha^k \mu$ and $\alpha^k = \frac{1}{\mu (k+k_0)}$, in which $k_0$ is sufficiently large, we obtain
\begin{align}
    \label{g-11}
    &\frac{1}{N} \sum_{n \in \N}   \E \|\vx_n^{k+1} - \vx_n'^{k+1} \|^2 \\
    \leq &    (1-\frac{1}{k+k_0})  \frac{1}{N}  \sum_{n \in \N}   \E \|\vx_n^{k} - \vx_n'^{k} \|^2 \nonumber \\
     & +  \frac{ 4   \E \| \nabla f(\vx^{k}_i; \xi_{i,z})\|^2 }{\mu^2 N Z (k+k_0)^2} +   \frac{ 4   \E \| \nabla f(\vx^{k}_i; \xi_{i,z})\|^2 }{\mu^2 N Z^2 (k+k_0)}.  \nonumber
\end{align}

Using telescopic cancellation on \eqref{g-11} from time $0$ to $k$, we deduce that
\begin{align}
    \label{g-12}
    &\frac{1}{N} \sum_{n \in \N} \E \|\vx_n^{k} - \vx_n'^{k} \|^2  \\
    \leq & \sum_{k'=0}^{k-1} \lp \prod_{t=k'+1}^{k-1}(1-\frac{1}{t+k_0})  \rp  \cdot \frac{ 4  \E \| \nabla f(\vx^{k'}_i; \xi_{i,z})\|^2 }{\mu^2 N Z (k'+k_0)^2} \nonumber  \\
      &   + \sum_{k'=0}^{k-1} \lp \prod_{t=k'+1}^{k-1}(1-\frac{1}{t+k_0})  \rp  \cdot
         \frac{ 4   \E\| \nabla f(\vx^{k'}_i; \xi_{i,z})\|^2 }{\mu^2 N Z^2 (k'+k_0)}     \nonumber \\
    \leq & \sum_{k'=0}^{k-1} \frac{k'+k_0}{k+k_0-1}  \cdot  \frac{ 4  \E \| \nabla f(\vx^{k'}_i; \xi_{i,z})\|^2 }{\mu^2 N Z (k'+k_0)^2}   \nonumber \\
    & +  \sum_{k'=0}^{k-1} \frac{k'+k_0}{k+k_0-1}  \cdot   \frac{ 4   \E\| \nabla f(\vx^{k'}_i; \xi_{i,z})\|^2 }{\mu^2 N Z^2 (k'+k_0)}   \nonumber  \\
    \leq &   \frac{ 4 \sum_{k'=0}^{k-1} \lp \E \| \nabla f(\vx^{k'}_i; \xi_{i,z})\|^2  / (k'+k_0) \rp }{\mu^2 N Z (k+k_0-1)} \nonumber \\
    & +   \frac{ 4 \sum_{k'=0}^{k-1}  \E \| \nabla f(\vx^{k'}_i; \xi_{i,z})\|^2 }{\mu^2 N Z^2 (k+k_0-1) }.  \nonumber
\end{align}

According to \eqref{g-5} and Definition \ref{def-sta},  the on-average stability can be upper bounded by
\begin{align}
    \label{g-12-1}
    & \frac{1}{NZ}\sum_{i \in \N}\sum_{z=1}^{Z}  \|\bar\vx'^{k}_{(i,z)} - \bar\vx^{k} \|^2   \\
    \leq &   \frac{ 4 \sum_{k'=0}^{k-1} \sum_{i \in \N}\sum_{z=1}^{Z} \lp \E \| \nabla f(\vx^{k'}_i; \xi_{i,z})\|^2  / (k'+k_0) \rp }{\mu^2 N^2 Z^2 (k+k_0-1)} \nonumber \\
    & +   \frac{ 4 \sum_{k'=0}^{k-1} \sum_{i \in \N}\sum_{z=1}^{Z}  \E \| \nabla f(\vx^{k'}_i; \xi_{i,z})\|^2 }{\mu^2 N^2 Z^3 (k+k_0-1) }.   \nonumber
\end{align}
This completes the proof.
\end{proof}

\section{Proof of Theorem \ref{the-oe}}
\begin{proof}
Regarding the distance between the algorithm output of attack-free DSGD at time $k+1$ and the optimal solution $\vx_{\cS}^*$ of the global empirical loss, if Assumption \ref{assumption:variance} holds, we have
\begin{align}
    \label{g-22}
    & \E \|\bar\vx^{k+1} - \vx_{\cS}^* \|^2 \\
    = & \E \| \frac{1}{N} \sum_{n \in \N} \sum_{m \in \N} w_{nm} (\vx_m^k - \alpha^k \nabla f(\vx_m^k,\xi_m^k)) - \vx_{\cS}^* \|^2 \nonumber \\
    = & \E \| \frac{1}{N} \sum_{n \in \N} (\vx_n^k - \alpha^k \nabla f(\vx_n^k,\xi_n^k)) - \vx_{\cS}^* \|^2 \nonumber \\
    = & \E \|\bar \vx^k - \frac{\alpha^k}{N} \sum_{n \in \N} \nabla f(\vx_n^k,\xi_n^k) - \vx_{\cS}^* \|^2 \nonumber \\
    \leq & \E \|\bar \vx^k - \frac{\alpha^k}{N} \sum_{n \in \N} \nabla F_{\cS_n}(\vx_n^k) - \vx_{\cS}^* \|^2 + \frac{(\alpha^k)^2 \sigma^2}{N}, \nonumber
\end{align}
where $\xi_n^k$ denotes the sample drawn from the local dataset $\cS_n$ at time $k$.

For the first term at the RHS of \eqref{g-22}, we have
\begin{align}
    \label{g-22-1}
    &  \|\bar \vx^k - \frac{\alpha^k}{N} \sum_{n \in \N} \nabla F_{\cS_n}(\vx_n^k) - \vx_{\cS}^* \|^2 \\
    = &   \|\bar \vx^k - \frac{\alpha^k}{N} \sum_{n \in \N} (\nabla F_{\cS_n}(\vx_n^k) -\nabla F_{\cS_n}(\vx_{\cS}^*))- \vx_{\cS}^* \|^2  \nonumber \\
    =&  \|\bar \vx^k - \vx_{\cS}^* \|^2 + (\alpha^k)^2\|\frac{1}{N} \sum_{n \in \N} (\nabla F_{\cS_n}(\vx_n^k) -\nabla F_{\cS_n}(\vx_{\cS}^*)) \|^2 \nonumber \\
    & - \frac{2 \alpha^k}{N} \sum_{n \in \N}  \langle \bar\vx^k - \vx_{\cS}^* , \nabla F_{\cS_n}(\vx_n^k) -\nabla F_{\cS_n}(\vx_{\cS}^*) \rangle. \nonumber
\end{align}
For the second term at the RHS of \eqref{g-22-1}, applying Assumption \ref{assumption:Lip}, we obtain
\begin{align}
    \label{g-22-2}
    &  \|\frac{1}{N} \sum_{n \in \N} (\nabla F_{\cS_n}(\vx_n^k) -\nabla F_{\cS_n}(\vx_{\cS}^*)) \|^2  \\
     \leq & 2 \|\frac{1}{N} \sum_{n \in \N} (\nabla F_{\cS_n}(\vx_n^k) - \nabla F_{\cS_n}(\bar\vx^k) )\|^2 \nonumber \\
     & +2 \|\frac{1}{N} \sum_{n \in \N} ( \nabla F_{\cS_n}(\bar\vx^k)- \nabla F_{\cS_n}(\vx_{\cS}^*)) \|^2 \nonumber \\
     \leq & 2 L^2 H^k  + 4L (F_{\cS}(\bar\vx^k)-F_{\cS}(\vx_{\cS}^*)),  \nonumber
\end{align}
where $H^k = \frac{1}{N} \sum_{n \in \N} \| \vx_n^k -\bar\vx^k\|^2$ represents the consensus error, as defined and analyzed in Lemma \ref{lemma-dm}.
For the third term at the RHS of \eqref{g-22-1}, we have
\begin{align}
    \label{g-22-3}
    &  \frac{2 \alpha^k}{N} \sum_{n \in \N}  \langle\bar\vx^k - \vx_{\cS}^* , \nabla F_{\cS_n}(\vx_n^k) -\nabla F_{\cS_n}(\vx_{\cS}^*) \rangle \\
    = & \frac{2 \alpha^k}{N} \sum_{n \in \N}  \langle\bar\vx^k - \vx_{\cS}^* , \nabla F_{\cS_n}(\vx_n^k)  \rangle  \nonumber \\
    = & \frac{2 \alpha^k}{N} \sum_{n \in \N}  \langle\bar\vx^k - \vx_{n}^k , \nabla F_{\cS_n}(\vx_n^k)  \rangle \nonumber \\
    & + \frac{2 \alpha^k}{N} \sum_{n \in \N}  \langle\vx_n^k - \vx_{\cS}^* , \nabla F_{\cS_n}(\vx_n^k)  \rangle \nonumber \\
    \geq & \frac{2 \alpha^k}{N} \sum_{n \in \N} \lp F_{\cS_n}(\bar\vx^k ) - F_{\cS_n}(\vx_{n}^k) -\frac{L}{2} \|\bar\vx^k  - \vx_{n}^k\|^2  \rp  \nonumber \\
    & + \frac{2 \alpha^k}{N} \sum_{n \in \N} \lp F_{\cS_n}(\vx_{n}^k ) - F_{\cS_n}(\vx_{\cS}^*) + \frac{\mu}{2} \| \vx_{n}^k-\vx_{\cS}^*\|^2  \rp  \nonumber \\
    \geq & \frac{2 \alpha^k}{N} \sum_{n \in \N} \lp F_{\cS_n}(\bar\vx^k ) - F_{\cS_n}(\vx_{\cS}^*) \rp  \nonumber \\
    & - \alpha^k L H^k + \alpha^k \mu \|\bar\vx^k -\vx_{\cS}^* \|^2  \nonumber \\
    \geq & 2 \alpha^k \lp F_{\cS}(\bar\vx^k ) - F_{\cS}(\vx_{\cS}^*) \rp - \alpha^k L H^k + \alpha^k \mu \|\bar\vx^k -\vx_{\cS}^* \|^2 . \nonumber
\end{align}
Substituting \eqref{g-22-2}--\eqref{g-22-3} into \eqref{g-22-1}, we have

\begin{align}
    \label{g-22-4}
    &  \|\bar \vx^k - \frac{\alpha^k}{N} \sum_{n \in \N} \nabla F_{\cS_n}(\vx_n^k) - \vx_{\cS}^* \|^2 \\
    \leq & (1-\mu\alpha^k) \|\bar \vx^k - \vx_{\cS}^* \|^2  + (2 (\alpha^k)^2 L^2 + \alpha^k L) H^k  \nonumber \\
    & + 2\alpha^k(2L \alpha^k -1) (F_{\cS}(\bar\vx^k)-F_{\cS}(\vx_{\cS}^*)). \nonumber
\end{align}

If we choose $\alpha^k \leq \frac{1}{2L}$, substituting \eqref{g-22-4} into \eqref{g-22}, we have
 \begin{align}
    \label{g-22-5}
    & \E \|\bar\vx^{k+1} - \vx_{\cS}^* \|^2 \\
    \leq & \E \|\bar \vx^k - \frac{\alpha^k}{N} \sum_{n \in \N} \nabla F_{\cS_n}(\vx_n^k) - \vx_{\cS}^* \|^2 + \frac{(\alpha^k)^2 \sigma^2}{N} \nonumber  \\
    \leq & (1-\mu\alpha^k) \E\|\bar \vx^k - \vx_{\cS}^* \|^2  + 2 \alpha^k L \E H^k + \frac{(\alpha^k)^2 \sigma^2}{N}. \nonumber
\end{align}

According to Lemma \ref{lemma-dm}, by setting the step size as $\alpha^k = \frac{1}{\mu(k+k_0)}$, in which $k_0$ sufficiently large, we obtain
\begin{align}
    \label{g-30}
     \E \|\bar\vx^{k+1} - \vx_{\cS}^* \|^2 \leq  & (1-\frac{1}{k+k_0}) \E \|\bar \vx^k - \vx_{\cS}^* \|^2 \\
   &  +  \frac{\sigma^2}{\mu^2N(k+k_0)^2}  +   \frac{2 c_1 L (\sigma^2 + \delta^2)}{\mu^3(k+k_0)^3}, \nonumber
\end{align}
where $c_1>0$ is a constant as defined in Lemma \ref{lemma-dm}.

Using telescopic cancellation on \eqref{g-30} from time $0$ to $k-1$, we deduce that
\begin{align}
    \label{g-31}
    & \E \|\bar\vx^{k} - \vx_{\cS}^* \|^2 \\
      \leq & \frac{k_0-1}{k+k_0-1} \|\bar \vx^0 - \vx_{\cS}^* \|^2 \nonumber \\
      &+  \sum_{k'=0}^{k-1} \frac{k'+k_0}{k+k_0-1} \bigg( \frac{\sigma^2}{\mu^2N(k'+k_0)^2}
      +  \frac{2 c_1 L (\sigma^2 + \delta^2)}{\mu^3(k'+k_0)^3} \bigg)  \nonumber \\
    \leq & \frac{(k_0-1)\|\bar \vx^0 - \vx_{\cS}^* \|^2}{k+k_0-1}   +    \frac{\sigma^2 ln (k+k_0-1)}{\mu^2N(k+k_0-1)} +   \frac{2 c_1 L (\sigma^2 + \delta^2)}{\mu^3 (k+k_0-1)} . \nonumber
\end{align}

Finally, combining \eqref{g-31} with Assumption \ref{assumption:Lip}, we obtain
\begin{align}
\label{g-31-1}
&    \E (F_{\cS}(\bar\vx^{k}) - F_{\cS} (\vx_{\cS}^*))  \\
\leq &  \frac{L}{2} \E \|\bar\vx^{k} - \vx_{\cS}^* \|^2 \nonumber\\
 \leq & \frac{(k_0-1)L\|\bar \vx^0 - \vx_{\cS}^* \|^2}{2(k+k_0-1)}   +    \frac{\sigma^2 L ln (k+k_0-1)}{2\mu^2N(k+k_0-1)} \nonumber  \\
 & +   \frac{ c_1 L^2 (\sigma^2 + \delta^2)}{\mu^3 (k+k_0-1)} . \nonumber
\end{align}
This completes the proof.
\end{proof}

\section{Proof of Theorem \ref{the-ge}}
\begin{proof}
According to Lemma \ref{lemma-sta}, the generalization error can be bounded by the sum of the on-average stability and the optimization error. Here we first demonstrate that the term $\E \| \nabla f(\vx^{k}_i; \xi_{i,z})\|^2$  in the on-average stability bound is inherently related to the optimization error under Assumption \ref{assumption:Lip}.
\begin{align}
    \label{g-14}
     &  \E \frac{1}{NZ}\sum_{i \in \N} \sum_{z=1}^{Z}  \| \nabla f(\vx^{k}_i; \xi_{i,z})\|^2   \\
    \leq &  4\E  \frac{1}{NZ}\sum_{i \in \N} \sum_{z=1}^{Z} \| \nabla f(\vx^{k}_i; \xi_{i,z}) - \nabla F_{\cS_{i}}(\vx^{k}_i)\|^2 \nonumber \\
    & + 4\E \frac{1}{N}\sum_{i \in \N} \|  \nabla F_{\cS_{i}}(\vx^{k}_i)- \nabla F_{\cS}(\vx^{k}_i)\|^2 \nonumber \\
    &  +  4\E \frac{1}{N}\sum_{i \in \N} \|  \nabla F_{\cS}(\vx^{k}_i) - \nabla F_{\cS}(\bar\vx^{k})\|^2   \nonumber \\
    & + 4\E \frac{1}{N}\sum_{i \in \N}  \| \nabla F_{\cS}(\bar\vx^{k})\|^2 \nonumber \\
    \leq &    4\E \frac{1}{N}\sum_{i \in \N} \|  \nabla F_{\cS}(\vx^{k}_i) - \nabla F_{\cS}(\bar\vx^{k})\|^2   \nonumber \\
    & + 4\E  \| \nabla F_{\cS}(\bar\vx^{k})\|^2 + 4 \sigma^2 + 4\delta^2 . \nonumber
\end{align}

For the First term at the RHS of \eqref{g-14}, according to Assumption \ref{assumption:Lip}, we have
\begin{align}
\label{g-16}
      \E \frac{1}{N}\sum_{i \in \N} \|  \nabla F_{\cS}(\vx^{k}_i) - \nabla F_{\cS}(\bar\vx^{k})\|^2
    \leq  L^2 \E H^k.
\end{align}
For the second term at the RHS of \eqref{g-14}, according to Assumption \ref{assumption:Lip}, we have
\begin{align}
\label{g-21}
 \E \| \nabla F_{\cS}(\bar\vx^{k}) \| ^2
\leq   2L \E (F_{\cS}(\bar\vx^{k}) - F_{\cS} (\vx_{\cS}^*)) .
\end{align}
Substituting \eqref{g-16} and \eqref{g-21} into \eqref{g-14}, we have
\begin{align}
    \label{g-21-1}
     &  \E \frac{1}{NZ}\sum_{i \in \N} \sum_{z=1}^{Z}  \| \nabla f(\vx^{k}_i; \xi_{i,z})\|^2   \\
    \leq & 4 L^2 \E H^k + 8L \E (F_{\cS}(\bar\vx^{k}) - F_{\cS} (\vx_{\cS}^*)) + 4 \sigma^2 + 4\delta^2. \nonumber
\end{align}



According to the Theorem \ref{the-oe} and Lemma \ref{lemma-dm},
substituting \eqref{g-31-1} into \eqref{g-12-1}, if setting $k \asymp \mu N Z$, we have
\begin{align}
\label{g-31-3}
       \frac{1}{NZ}\sum_{i \in \N}\sum_{z=1}^{Z} \E \|\bar\vx'^{k}_{(i,z)} - \bar\vx^{k} \|^2
     \leq  O\lp  \frac{\sigma^2 +\delta^2 }{ \mu N  Z} \rp.
\end{align}

According to Lemma \ref{lemma-sta}, let $v = L$ and $k \asymp \mu N Z$. Then by substituting \eqref{g-31-1} and \eqref{g-31-3} into \eqref{lemma}, we have
\begin{align}
\label{g-31-6}
    &\E[F(\bar \vx^k)-F_{\cS}(\bar \vx^k)] \\
   \leq & O(\frac{\|\bar\vx^0 - \vx_{\cS}^*\|^2}{ \mu N Z})+  \tilde O(\frac{\sigma^2}{\mu N Z}) + O(\frac{\delta^2}{\mu N Z}). \nonumber
\end{align}
This completes the proof.
\end{proof}

\section{Proof of Theorem \ref{the-onsta-b}}

\begin{proof}
Below, following \cite{ye2024ge,ye2024generalization}, we denote the output of Byzantine-resilient DGSD on datasets $\cS$ and $\cS^{(i,z)}$ at time $k$ as  $\mathcal{L}(\cS)= $ $\bar\vx^k =\frac{1}{N}\sum_{n \in \N}\vx_n^k$ and $\cL(\cS^{(i,z)})=\bar\vx'^k_{(i,z)} = \frac{1}{N}\sum_{n \in \N}\vx'^k_n$, respectively. Here, $\vx'^k_n$ denotes the local model of non-Byzantine agent $n$ at time $k$ when trained on the perturbed global dataset $\cS^{(i,z)} = \{\cS_{m} | m \in \mathcal{N},  m \neq i\} \cup \cS_i^{(z)}$, and $\bar\vx'^k_{(i,z)}$ is the average of $\vx'^k_n$.
We analyze the stability of Algorithm \ref{robust-DSGD} under two scenarios, depending on whether the virtual mixing matrix is doubly stochastic or row stochastic.

\subsection{Row Stochastic Virtual Mixing Matrix}
When the virtual mixing matrix is row stochastic, for any $v \in (0,1)$, we decompose the distance between $\bar\vx^{k+1}$ and $\bar\vx'^{k+1}_{(i,z)}$ into three parts as

\begin{align}
\label{ger-1}
&  \|\bar\vx^{k+1}-\bar\vx'^{k+1}_{(i,z)}\|^2
\leq  \frac{2}{v}\|\bar\vx^{k+1} - \bar\vx^{k+\frac{1}{2}} \|^2  \\
&+  \frac{1}{1-v} \|\bar\vx^{k+\frac{1}{2}} - \bar\vx'^{k+\frac{1}{2}}_{(i,z)} \|^2
 + \frac{2}{v} \|\bar\vx'^{k+1}_{(i,z)} - \bar\vx'^{k+\frac{1}{2}}_{(i,z)} \|^2. \nonumber
\end{align}

 For the first term at the RHS of (\ref{ger-1}), we bound it by
\begin{align}
    \label{ger-2}
    &\|\bar\vx^{k+1} - \bar\vx^{k+\frac{1}{2}} \|^2 \\
=& \| \frac{1}{N}\sum_{n \in \N}\vx_n^{k+1} - \frac{1}{N}\sum_{n \in \N}\vx_n^{k+\frac{1}{2}} \|^2 \nonumber\\
\leq& 2\| \frac{1}{N} \sum_{n \in \N}(\vx_n^{k+1} - \hat\vx_n^{k+\frac{1}{2}}) \|^2 \nonumber \\
&+ 2\| \frac{1}{N} \sum_{n \in \N} (\hat\vx_n^{k+\frac{1}{2}} - \vx_n^{k+\frac{1}{2}})\|^2. \nonumber
\end{align}
For the first term at the RHS of \eqref{ger-2}, we use the contraction property of the robust aggregation rules $ \{ \A_n \}_{n \in  \N}$ in Definition \ref{definition:mixing-matrix} to derive
\begin{align}
\label{ger-2-1}
    &\frac{1}{N} \sum_{n \in \N} \| \vx_n^{k+1} - \hat\vx_n^{k+\frac{1}{2}} \|^2 \\
    \leq &  \frac{1}{N} \sum_{n \in \N} \rho^2 \max_{m \in \N_n \cup \{n\}} \| \vx_{m}^{k+\frac{1}{2}} - \hat\vx_n^{k+\frac{1}{2}} \|^2 \nonumber \\
    \leq &  \frac{1}{N} \sum_{n \in \N} \rho^2 \max_{m \in \N} \| \vx_{m}^{k+\frac{1}{2}} - \hat\vx_n^{k+\frac{1}{2}} \|^2 \nonumber \\
    \leq &  \frac{1}{N} \sum_{n \in \N} \rho^2 \lp 2\max_{m \in \N} \| \vx_{m}^{k+\frac{1}{2}} - \bar\vx^{k+\frac{1}{2}} \|^2 + 2\|  \bar\vx^{k+\frac{1}{2}} - \hat\vx_n^{k+\frac{1}{2}}\|^2 \rp \nonumber \\
    \leq & 4\rho^2 \frac{1}{N} \sum_{n \in \N}  \max_{m \in \N} \| \vx_{m}^{k+\frac{1}{2}} - \bar\vx^{k+\frac{1}{2}} \|^2 \nonumber \\
     \leq & 4\rho^2 \sum_{n \in \N}   \| \vx_{n}^{k+\frac{1}{2}} - \bar\vx^{k+\frac{1}{2}} \|^2. \nonumber
\end{align}
For the second term at the RHS of (\ref{ger-2}), it holds that
\begin{align}
    \label{ger-4}
    &\| \frac{1}{N} \sum_{n \in \N} (\hat\vx_n^{k+\frac{1}{2}} - \vx_n^{k+\frac{1}{2}})\|^2 \\
     =&  \lnorm \frac{1}{N} \bm{1}^\top (WX^{k+\frac{1}{2}}-\frac{1}{N}\bm{1}\bm{1}^{\top}X^{k+\frac{1}{2}})
        \rnorm^2
          \nonumber\\
        =& \frac{1}{N^2}  \lnorm (\bm{1}^\top W-\bm{1}^{\top})
        (X^{k+\frac{1}{2}}-\frac{1}{N}\bm{1}\bm{1}^{\top}X^{k+\frac{1}{2}})
        \rnorm^2
         \nonumber\\
        \le& \frac{1}{N^2} \lnorm W^\top\bm{1} -\bm{1}
        \rnorm^2
        \lnorm X^{k+\frac{1}{2}}-\frac{1}{N}\bm{1}\bm{1}^{\top}X^{k+\frac{1}{2}}
        \rnorm_F^2
          \nonumber\\
		=& \frac{\chi^2}{N}\sum_{n\in\N} \|  \vx^{k+\frac{1}{2}}_{n}-\bar\vx^{k+\frac{1}{2}}\|^2,   \nonumber
\end{align}
in which  $X^{k+\frac{1}{2}} = [\vx^{k+\frac{1}{2}}_1, \cdots, \vx^{k+\frac{1}{2}}_{N}]^\top \! \in \mathbb{R}^{N \times d}$.
Substituting \eqref{ger-4} and \eqref{ger-2-1} into \eqref{ger-2}, and applying Lemma 2 in \cite{wu2022byzantine}, we have
\begin{align}
    \label{ger-5}
    &\|\bar\vx^{k+1} - \bar\vx^{k+\frac{1}{2}} \|^2\\
 \leq &  8 \rho^2 \sum_{n \in \N} \|\vx_n^{k+\frac{1}{2}}  - \bar\vx^{k+\frac{1}{2}} \|^2 + \frac{2 \chi^2}{N}\sum_{n\in\N} \|  \vx^{k+\frac{1}{2}}_{n}-\bar\vx^{k+\frac{1}{2}}\|^2\nonumber \\
\leq &  (8\rho^2 N+2 \chi^2)  (2+12(\alpha^k)^2 L^2) H^k \nonumber \\
& + (64 \rho^2 N+16 \chi^2) (\alpha^k)^2 (\sigma^2 + \delta^2).  \nonumber
\end{align}
Note that for the third term at the RHS of (\ref{ger-1}), an inequality similar to (\ref{ger-5}) also holds true.

Now, we analyze the second term at the RHS of (\ref{ger-1}). For the non-Byzantine agent $i$, the probability of Algorithm \ref{robust-DSGD} selecting the same training sample from  $\cS_i$ and $\cS_i^{(z)}$ at time $k$ is $1-\frac{1}{Z}$.  According to Lemma 6 in \cite{sun2021stability}, for any $p>0$, if the loss is strongly convex and Assumption \ref{assumption:Lip} holds, then choosing $\alpha^k \leq \frac{1}{L}$ yields
\begin{align}
    \label{ger-7}
           &   \|\bar\vx^{k+\frac{1}{2}} - \bar\vx'^{k+\frac{1}{2}}_{(i,z)} \|^2  \\
     = &     \|\bar\vx^k -\frac{\alpha^k}{N} \sum_{m \in \N} \nabla f(\vx^{k}_m; \xi_m^{k}) \nonumber  \\
     & \hspace{2em}-\bar\vx'^k_{(i,z)} +\frac{\alpha^k}{N} \sum_{m \in \N} \nabla f(\vx'^{k}_m; \xi_m^{k}) \|^2  \nonumber \\
    \leq& (1+p)  \|\bar\vx^k -\frac{\alpha^k}{N} \sum_{m \in \N} \nabla f(\bar\vx^{k}; \xi_m^{k} ) \nonumber \\
    & \hspace{4em} -\bar\vx'^k_{(i,z)} + \frac{\alpha^k}{N} \sum_{m \in \N} \nabla f(\bar\vx'^{k}_{(i,z)}; \xi_m^{k}) \|^2  \nonumber  \\
    & +  2(1+p^{-1})\frac{(\alpha^k)^2}{N} \sum_{m \in \N}  \|  \nabla f(\bar\vx^{k}; \xi_m^{k}) -  \nabla f(\vx^{k}_m; \xi_m^{k}) \|^2 \nonumber  \\
    & + 2(1+p^{-1}) \frac{(\alpha^k)^2}{N} \sum_{m \in \N} \|  \nabla f(\bar\vx'^{k}_{(i,z)}; \xi_m^{k}) - \nabla f(\vx'^{k}_m; \xi_m^{k}) \|^2   \nonumber  \\
    \leq&    (1-\alpha^k\mu)^2 (1+p) \| \bar\vx^k - \bar\vx'^k_{(i,z)}\|^2
     + 4 (1+p^{-1}) (\alpha^k)^2 L^2 H^k. \nonumber
\end{align}
Substituting \eqref{ger-7} and \eqref{ger-5} into \eqref{ger-1}, in this case we have
\begin{align}
\label{ger-8}
&  \E\|\bar\vx^{k+1}-\bar\vx'^{k+1}_{(i,z)}\|^2 \\
\leq &  \frac{(1-\alpha^k\mu)^2 (1+p)}{1-v}  \E \| \bar\vx^k - \bar\vx'^k_{(i,z)}\|^2   \nonumber \\
     & + \frac{4 (1+p^{-1}) (\alpha^k)^2 L^2 \E H^k}{1-v} \nonumber \\
    & + \frac{4}{v}(8\rho^2 N+2 \chi^2)  (2+12(\alpha^k)^2 L^2) \E H^k \nonumber \\
& + \frac{4}{v}(64 \rho^2 N+16 \chi^2) (\alpha^k)^2 (\sigma^2 + \delta^2).  \nonumber
\end{align}

For non-Byzantine agent $i$, the probability of Algorithm \ref{robust-DSGD} selecting the different training samples from  $\cS_i$ and $\cS_i^{(z)}$ at time $k$ is $\frac{1}{Z}$. According to Lemma 6 in \cite{sun2021stability}, for any $p>0$, if the loss is strongly convex and Assumption \ref{assumption:Lip} holds, then choosing $\alpha^k \leq \frac{1}{L}$ yields
\begin{align}
    \label{ger-9}
   & \|\bar\vx^{k+\frac{1}{2}} - \bar\vx'^{k+\frac{1}{2}}_{(i,z)} \|^2   \\
    \leq& (1+p)   \|\bar\vx^k -\frac{\alpha^k}{N} \sum_{m \in \N} \nabla f(\bar\vx^{k}; \xi_m^{k} ) \nonumber \\
    & \hspace{4em} -\bar\vx'^k_{(i,z)} + \frac{\alpha^k}{N} \sum_{m \in \N} \nabla f(\bar\vx'^{k}_{(i,z)}; \xi_m^{k}) \|^2  \nonumber  \\
    & +  3(1+p^{-1})   \| \frac{\alpha^k}{N} \lp  \nabla f(\vx^{k}_i; \xi'_{i,z}) - \nabla f(\vx^{k}_i; \xi_{i,z}) \rp \|^2 \nonumber \\
    & +  3(1+p^{-1}) \frac{(\alpha^k)^2}{N} \sum_{m \in \N}  \| \nabla f(\bar\vx^{k}; \xi_m^{k}) - \nabla f(\vx^{k}_m; \xi_m^{k}) \|^2 \nonumber  \\
    & + 3(1+p^{-1}) \frac{(\alpha^k)^2}{N} \sum_{m \in \N} \|  \nabla f(\bar\vx'^{k}_{(i,z)}; \xi_m^{k}) - \nabla f(\vx'^{k}_m; \xi_m^{k}) \|^2   \nonumber  \\
    \leq&    (1-\alpha^k\mu)^2 (1+p)   \| \bar\vx^k - \bar\vx'^k_{(i,z)}\|^2
     + 6 (1+p^{-1}) (\alpha^k)^2 L^2 H^k \nonumber \\
    & + \frac{3(1+p^{-1})(\alpha^k)^2  \|   \nabla f(\vx^{k}_i; \xi'_{i,z}) - \nabla f(\vx^{k}_i; \xi_{i,z})  \|^2  }{N^2}  . \nonumber
\end{align}
Substituting \eqref{ger-9} and \eqref{ger-5} into \eqref{ger-1}, in this case we have
\begin{align}
\label{ger-10}
&   \E \|\bar\vx^{k+1}-\bar\vx'^{k+1}_{(i,z)}\| \\
\leq &  \frac{(1-\alpha^k\mu)^2 (1+p)}{1-v}  \E \| \bar\vx^k - \bar\vx'^k_{(i,z)}\|^2 \nonumber \\
& + \frac{12(1+p^{-1})(\alpha^k)^2}{N^2 (1-v)}   \E \|    \nabla f(\vx^{k}_i; \xi_{i,z})  \|^2  \nonumber \\
    & + \frac{6 (1+p^{-1}) (\alpha^k)^2 L^2 \E H^k}{1-v} \nonumber \\
    & + \frac{4}{v}(8\rho^2 N+2 \chi^2)  (2+12(\alpha^k)^2 L^2) \E H^k \nonumber \\
& + \frac{4}{v}(64 \rho^2 N+16 \chi^2) (\alpha^k)^2 (\sigma^2 + \delta^2). \nonumber
\end{align}

Combining the two cases of \eqref{ger-8} and \eqref{ger-10}, according to Definition \ref{def-sta}, we have
\begin{align}
\label{ger-11}
&  \frac{1}{NZ}\sum_{i \in \N}\sum_{z=1}^{Z} \E \|\bar\vx^{k+1}-\bar\vx'^{k+1}_{(i,z)}\| \\
\leq &  \frac{(1-\alpha^k\mu)^2 (1+p)}{1-v}  \frac{1}{NZ}\sum_{i \in \N}\sum_{z=1}^{Z} \E \| \bar\vx^k - \bar\vx'^k_{(i,z)}\|^2 \nonumber \\
& + \frac{12(1+p^{-1})(\alpha^k)^2}{N^2 Z (1-v)} \frac{1}{NZ}\sum_{i \in \N}\sum_{z=1}^{Z}  \E \|    \nabla f(\vx^{k}_i; \xi_{i,z})  \|^2  \nonumber \\
    & + \frac{6 (1+p^{-1}) (\alpha^k)^2 L^2 \E H^k}{1-v} \nonumber \\
    & + \frac{4}{v}(8\rho^2 N+2 \chi^2)  (2+12(\alpha^k)^2 L^2) \E H^k \nonumber \\
& + \frac{4}{v}(64 \rho^2 N+16 \chi^2) (\alpha^k)^2 (\sigma^2 + \delta^2).  \nonumber
\end{align}

If we set $v = \frac{\alpha^k\mu}{2}$, $p =\alpha^k \mu$ and $\alpha^k = \frac{2}{\mu (k+k_1)}$, according to Lemma \ref{lemma-dm-b},
we have
\begin{align}
\label{ger-12}
     & \frac{1}{NZ}\sum_{i \in \N}\sum_{z=1}^{Z} \E \|\bar\vx^{k+1}-\bar\vx'^{k+1}_{(i,z)}\|^2 \\
     \leq &  (1-\frac{1}{k+k_1})  \frac{1}{NZ}\sum_{i \in \N}\sum_{z=1}^{Z} \E  \| \bar\vx^k - \bar\vx'^k_{(i,z)}\|^2 \nonumber \\
     & + \frac{96}{\mu^2 (k+k_1)^2 N^2 Z } \frac{1}{NZ}\sum_{i \in \N}\sum_{z=1}^{Z}  \E \|    \nabla f(\vx^{k}_i; \xi_{i,z})  \|^2  \nonumber \\
     & + \frac{48}{\mu^2 (k+k_1) N^2 Z } \frac{1}{NZ}\sum_{i \in \N}\sum_{z=1}^{Z}  \E \|    \nabla f(\vx^{k}_i; \xi_{i,z})  \|^2  \nonumber \\
     & + \frac{48 L^2 c_2 (\sigma^2 + \delta^2) }{\mu^4 (k+k_1)^4} + \frac{24 L^2 c_2 (\sigma^2 + \delta^2) }{\mu^4 (k+k_1)^3} \nonumber \\ &+    \frac{16 c_2 (4\rho^2 N+\chi^2) (\sigma^2 + \delta^2)}{\mu^2 (k+k_1)}  \nonumber \\
     & + \frac{384 c_2 L^2 (4\rho^2 N+\chi^2) (\sigma^2 + \delta^2) }{\mu^4 (k+k_1)^3}    \nonumber \\
     & + \frac{256 (4\rho^2N +\chi^2)  (\sigma^2 + \delta^2) }{\mu^2 (k+k_1)}, \nonumber
\end{align}
where $c_2>0$ is a constant defined in Lemma \ref{lemma-dm-b}.


Using telescopic cancellation on \eqref{ger-12} from time $0$ to $k$, we deduce that

\begin{align}
    \label{ger-13}
    & \frac{1}{NZ}\sum_{i \in \N}\sum_{z=1}^{Z} \E \|\bar\vx^{k}-\bar\vx'^{k}_{(i,z)}\|^2 \\
    \leq &   \frac{(16c_2 +256) (4\rho^2 N+\chi^2) (\sigma^2 + \delta^2)}{\mu^2} \nonumber  \\
    & + \frac{384 c_2 L^2  (4\rho^2 N+\chi^2) (\sigma^2 + \delta^2) }{\mu^4 (k+k_1-1)}  + \frac{72 L^2 c_2 (\sigma^2 + \delta^2) }{\mu^4 (k+k_1-1)} \nonumber \\
   & + \frac{96 \sum_{k'=0}^{k-1} \sum_{i \in \N}\sum_{z=1}^{Z} \E  \lp \| \nabla f(\vx^{k'}_i; \xi_{i,z}) \|^2 / (k'+k_1) \rp }{\mu^2 N^3 Z^2 (k+k_1-1) }  \nonumber \\
    & + \frac{48 \sum_{k'=0}^{k-1} \sum_{i \in \N}\sum_{z=1}^{Z} \E  \| \nabla f(\vx^{k'}_i; \xi_{i,z}) \|^2  }{\mu^2 N^3 Z^2  (k+k_1-1)}.   \nonumber
\end{align}

\subsection{Doubly Stochastic Virtual Mixing Matrix}
When the virtual mixing matrix becomes doubly stochastic, it is convenient to analyze $ \frac{1}{N} \sum_{n \in \N} \|\vx_n^{k} - \vx'^{k}_n \|^2$ instead of $\|\bar\vx'^{k}_{(i,z)} - \bar\vx^{k} \|^2$, because $\frac{1}{N} \sum_{n \in \N} \|\vx_n^{k} - \vx'^{k}_n \|^2 \geq \|\bar\vx'^{k}_{(i,z)} - \bar\vx^{k} \|^2$. For any $v \in (0,1)$, we decompose $\frac{1}{N} \sum_{n \in \N} \|\vx_n^{k+1} - \vx'^{k+1}_n \|^2$ into three parts, as
\begin{align}
\label{gb-5}
    &  \frac{1}{N} \sum_{n \in \N} \|\vx_n^{k+1} - \vx'^{k+1}_n \|^2 \\
\leq & \frac{2}{v}\frac{1}{N} \sum_{n \in \N}    \|\vx_n^{k+1} - \hat\vx_n^{k+\frac{1}{2}} \|^2 + \frac{2}{v}\frac{1}{N} \sum_{n \in \N}  \|\vx_n'^{k+1} - \hat\vx_n'^{k+\frac{1}{2}} \|^2  \nonumber\\
&+ \frac{1}{1-v} \frac{1}{N} \sum_{n \in \N}  \|\hat\vx_n^{k+\frac{1}{2}} - \hat\vx_n'^{k+\frac{1}{2}} \|^2.  \nonumber
\end{align}

For the first and second terms at the RHS of (\ref{gb-5}), we can bound them similarly to \eqref{ger-5}.

\begin{align}
\label{gb-5-8}
    &\frac{1}{N} \sum_{n \in \N} \| \vx_n^{k+1} - \hat\vx_n^{k+\frac{1}{2}} \|^2 \\
     \leq & 4\rho^2 \sum_{n \in \N}   \| \vx_{n}^{k+\frac{1}{2}} - \bar\vx^{k+\frac{1}{2}} \|^2 \nonumber \\
     \leq &  4\rho^2 N  (2+12(\alpha^k)^2 L^2) H^k + 32 \rho^2 N (\alpha^k)^2 (\sigma^2 + \delta^2).   \nonumber
\end{align}

Now, we analyze the third term at the RHS of (\ref{gb-5}). For the non-Byzantine agent $i$, the probability of Algorithm \ref{robust-DSGD} selecting the same training sample from  $\cS_i$ and $\cS_i^{(z)}$ at time $k$ is $1-\frac{1}{Z}$.  According to Lemma 6 in \cite{sun2021stability}, if the loss is strongly convex and Assumption \ref{assumption:Lip} holds, then choosing $\alpha^k \leq \frac{1}{L}$ yields
\begin{align}
    \label{gb-6}
    & \frac{1}{N} \sum_{n \in \N} \|\hat\vx_n^{k+\frac{1}{2}} - \hat\vx_n'^{k+\frac{1}{2}} \|^2  \\
    = & \frac{1}{N} \sum_{n \in \N} \|\sum_{m \in \N} w_{nm} \lp \vx_m^{k+\frac{1}{2}} - \vx_m'^{k+\frac{1}{2}} \rp\|^2 \nonumber\\
    \leq & \frac{1}{N} \sum_{n \in \N} \sum_{m \in \N} w_{nm}  \| \vx_m^{k+\frac{1}{2}} - \vx_m'^{k+\frac{1}{2}} \|^2 \nonumber\\
    \leq &   (1-\alpha^k\mu)^2 \frac{1}{N} \sum_{n \in \N} \| \vx_n^{k} - \vx_n'^{k} \|^2. \nonumber
\end{align}
Substituting (\ref{gb-5-8}) and (\ref{gb-6}) into (\ref{gb-5}), in this case we have
\begin{align}
\label{gb-7}
    & \frac{1}{N} \sum_{n \in \N} \E \|\vx_n^{k+1} - \vx'^{k+1}_n \|^2 \\
     \leq &  \frac{(1-\alpha^k\mu)^2}{1-v}   \frac{1}{N} \sum_{n \in \N} \E \|\vx_n^{k} - \vx'^{k}_n \|^2 \nonumber \\
     & + \frac{16}{v}\rho^2 N  (2+12(\alpha^k)^2 L^2) \E H^k  \nonumber \\
     & + \frac{128}{v} \rho^2 N (\alpha^k)^2 (\sigma^2 + \delta^2).  \nonumber
\end{align}

For non-Byzantine agent $i$, the probability of Algorithm \ref{robust-DSGD} selecting the different training samples from  $\cS_i$ and $\cS_i^{(z)}$ at time $k$ is $\frac{1}{Z}$. According to Lemma 6 in \cite{sun2021stability}, for any $p>0$, if the loss is strongly convex and Assumption \ref{assumption:Lip} holds, then choosing $\alpha^k \leq \frac{1}{L}$ yields
\begin{align}
    \label{gb-8}
    & \frac{1}{N} \sum_{n \in \N} \|\hat\vx_n^{k+\frac{1}{2}} - \hat\vx_n'^{k+\frac{1}{2}} \|^2 \\
   \leq &  \frac{(1-\alpha^k \mu)^2 (1+p)}{N}  \sum_{n \in \N} \|\vx_n^{k} - \vx_n'^{k} \|^2  \nonumber \\    
   & +  \frac{(\alpha^k)^2(1+p^{-1})}{N}  \| \nabla f(\vx^{k}_i; \xi_{i,z}) - \nabla f(\vx^{k}_i; \xi'_{i,z})\|^2.   \nonumber
\end{align}
Substituting (\ref{gb-5-8}) and (\ref{gb-8}) into (\ref{gb-5}), in this case we have
\begin{align}
\label{gb-9}
     & \frac{1}{N} \sum_{n \in \N} \E \|\vx_n^{k+1} - \vx'^{k+1}_n \|^2 \\
     \leq &  \frac{(1-\alpha^k\mu)^2(1+p)}{1-v}   \frac{1}{N} \sum_{n \in \N} \E \|\vx_n^{k} - \vx'^{k}_n \|^2 \nonumber \\
     & + \frac{16}{v}\rho^2 N  (2+12(\alpha^k)^2 L^2) \E H^k  \nonumber \\
     & + \frac{128}{v} \rho^2 N (\alpha^k)^2 (\sigma^2 + \delta^2)  \nonumber \\
     & +  \frac{4(\alpha^k)^2(1+p^{-1})}{N}  \E  \| \nabla f(\vx^{k}_i; \xi_{i,z}) \|^2. \nonumber
\end{align}

Combining the two cases of (\ref{gb-9}) and (\ref{gb-7}), we can obtain

\begin{align}
\label{gb-10}
    & \frac{1}{N} \sum_{n \in \N} \E \|\vx_n^{k+1} - \vx'^{k+1}_n \|^2 \\
    \leq &  \frac{(1-\alpha^k\mu)^2(1+\frac{p}{Z})}{1-v}   \frac{1}{N} \sum_{n \in \N} \E \|\vx_n^{k} - \vx'^{k}_n \|^2 \nonumber \\
     & + \frac{16}{v}\rho^2 N  (2+12(\alpha^k)^2 L^2) \E H^k \nonumber \\
     & + \frac{128}{v} \rho^2 N (\alpha^k)^2 (\sigma^2 + \delta^2)  \nonumber \\
     & +  \frac{4(\alpha^k)^2(1+p^{-1})}{NZ}  \E  \| \nabla f(\vx^{k}_i; \xi_{i,z}) \|^2. \nonumber
\end{align}

If we set $v = \frac{\alpha^k\mu}{2}$, $p = Z \alpha^k \mu$ and $\alpha^k = \frac{2}{\mu (k+k_1)}$, according to Lemma \ref{lemma-dm-b}, we have
\begin{align}
\label{gb-10-1}
 & \frac{1}{N} \sum_{n \in \N} \E \|\vx_n^{k+1} - \vx'^{k+1}_n \|^2  \\
     \leq &  (1-\frac{1}{k+k_1})  \frac{1}{N} \sum_{n \in \N} \E \|\vx_n^{k} - \vx'^{k}_n \|^2 \nonumber \\
     &+  \frac{2 (16 c_2 +256)  \rho^2 N (\sigma^2 + \delta^2)}{\mu^2 (k+k_1)}
      + \frac{768 c_2 L^2 \rho^2 N (\sigma^2 + \delta^2) }{\mu^4 (k+k_1)^3} \nonumber \\
      &
      + \frac{16 \E  \| \nabla f(\vx^{k}_i; \xi_{i,z}) \|^2  }{\mu^2 NZ (k+k_1)^2}
      + \frac{8 \E  \| \nabla f(\vx^{k}_i; \xi_{i,z}) \|^2  }{\mu^2 N Z^2 (k+k_1)}.  \nonumber
\end{align}

Using telescopic cancellation on \eqref{gb-10-1} from time $0$ to $k$, we deduce that
\begin{align}
    \label{gb-10-2}
     & \frac{1}{N} \sum_{n \in \N}  \|\vx_n^{k} - \vx_n'^{k} \|^2   \\
    \leq &   \frac{2 (16c_2 +256) \rho^2 N (\sigma^2 + \delta^2)}{\mu^2} + \frac{768 c_2 L^2 \rho^2 N (\sigma^2 + \delta^2) }{\mu^4 (k+k_1-1)}  \nonumber \\
   & + \frac{16 \sum_{k'=0}^{k-1}  \lp \E  \| \nabla f(\vx^{k'}_i; \xi_{i,z}) \|^2 / (k'+k_1)  \rp}{\mu^2 NZ (k+k_1-1)}  \nonumber \\
    & + \frac{8 \sum_{k'=0}^{k-1} \E  \| \nabla f(\vx^{k'}_i; \xi_{i,z}) \|^2  }{\mu^2 N Z^2  }.  \nonumber
\end{align}

According to \eqref{gb-10-2} and Definition \ref{def-sta},  the on-average stability can be upper bounded by
\begin{align}
    \label{gb-10-2-1}
    & \frac{1}{NZ}\sum_{i \in \N}\sum_{z=1}^{Z}  \|\bar\vx'^{k}_{(i,z)} - \bar\vx^{k} \|^2   \\
    \leq &   \frac{ 16 \sum_{k'=0}^{k-1} \sum_{i \in \N}\sum_{z=1}^{Z} \lp \E \| \nabla f(\vx^{k'}_i; \xi_{i,z})\|^2  / (k'+k_1) \rp }{\mu^2 N^2 Z^2 (k+k_1-1)} \nonumber \\
    & +   \frac{ 8 \sum_{k'=0}^{k-1} \sum_{i \in \N}\sum_{z=1}^{Z}  \E \| \nabla f(\vx^{k'}_i; \xi_{i,z})\|^2 }{\mu^2 N^2 Z^3 (k+k_1-1) }    \nonumber \\
   &+ \frac{2 (16c_2 +256) \rho^2 N (\sigma^2 + \delta^2)}{\mu^2} + \frac{384 c_2 L^2 2 \rho^2 N (\sigma^2 + \delta^2) }{\mu^4 (k+k_1-1)}.  \nonumber
\end{align}
This completes the proof.
\end{proof}

\section{Proof of Theorem \ref{the-oe-b}}
\begin{proof}
Regarding the distance between the algorithm output of Byzantine-resilient DSGD at time $k+1$ and the optimal solution $\vx_{\cS}^*$ of the global empirical loss, if Assumption \ref{assumption:variance} holds, for any $a \in (0,1)$, we have
\begin{align}
    \label{ger-17}
    & \E \|\bar\vx^{k+1} - \vx_{\cS}^* \|^2  \\
    \leq &  \frac{1}{1-a} \E \|\bar\vx^{k+1} - \bar\vx^{k+\frac{1}{2}} \|^2
     + \frac{1}{a} \E \|\bar\vx^{k+\frac{1}{2}} - \vx_{\cS}^* \|^2.  \nonumber
\end{align}
The first term at the RHS of \eqref{ger-17} can be bounded by \eqref{ger-5}.
For the second term at the RHS of \eqref{ger-17}, if we choose $\alpha^k \leq \frac{1}{2L}$, following \eqref{g-22-5},  we have
\begin{align}
    \label{gb-17-1}
    & \E \| \bar\vx^{k+\frac{1}{2}} - \vx_{\cS}^* \|^2 \\
    \leq & (1-\mu\alpha^k) \E \|\bar \vx^k - \vx_{\cS}^* \|^2  + 2 \alpha^k L \E H^k + \frac{(\alpha^k)^2 \sigma^2}{N}. \nonumber
\end{align}


Substituting \eqref{gb-17-1} and \eqref{ger-5} into \eqref{ger-17}, setting $a=1-\frac{\mu\alpha^k}{2}$, because $1-\mu\alpha^k \leq (1-\frac{\mu\alpha^k}{2})^2$, we have
\begin{align}
    \label{ger-19}
    & \E \|\bar\vx^{k+1} - \vx_{\cS}^* \|^2 \\
    \leq & (1-\frac{\alpha^k \mu}{2})  \E \|\bar \vx^k - \vx_{\cS}^* \|^2  +  \frac{ (\alpha^k)^2 \sigma^2}{\lp 1-\frac{\alpha^k \mu}{2}\rp N}
     +  \frac{2\alpha^k L \E  H^k}{1-\frac{\alpha^k \mu}{2}}  \nonumber\\
     & + \frac{4(4 \rho^2 N+ \chi^2) (2+12 (\alpha^k)^2 L^2) \E H^k}{\alpha^k \mu} \nonumber \\
     & + \frac{ 32(4   \rho^2 N +  \chi^2) \alpha^k (\sigma^2 + \delta^2)}{\mu}.  \nonumber
\end{align}

If we set $\alpha^k = \frac{2}{\mu (k+k_1)}$, then since $\frac{1}{a}=\frac{k+k_1}{k+k_1-1} \in (1,2)$, we have
\begin{align}
    \label{ger-20}
    & \E \|\bar\vx^{k+1} - \vx_{\cS}^* \|^2 \\
     \leq & (1-\frac{1}{k+k_1})  \E \|\bar \vx^k - \vx_{\cS}^* \|^2
     +  \frac{8\sigma^2}{\mu^2 N (k+k_1)^2}    \nonumber \\
     & + \frac{8 c_2 L (\sigma^2 + \delta^2)}{\mu^3 (k+k_1)^3}
      + \frac{96 c_2 L^2 (4\rho^2 N+\chi^2)  (\sigma^2 + \delta^2)}{ \mu^4 (k+k_1)^3} \nonumber \\
      & + \frac{(64+4c_2)  (4\rho^2 N + \chi^2) (\sigma^2 + \delta^2)}{ \mu^2 (k+k_1)}.  \nonumber
\end{align}

Using telescopic cancellation on \eqref{ger-20} from time $0$ to $k$, we deduce that
\begin{align}
    \label{ger-21}
    & \E \|\bar\vx^{k} - \vx_{\cS}^* \|^2 \\
     \leq & \frac{(k_1-1)  \|\bar \vx^0 - \vx_{\cS}^* \|^2}{k+k_1-1} + \frac{8\sigma^2 ln(k+k_1-1)}{ \mu^2 N (k+k_1-1)} \nonumber \\
     & + \frac{8 c_2 L (\sigma^2 + \delta^2)}{\mu^3 (k+k_1-1) }
      + \frac{96c_2 L^2 (4\rho^2 N+\chi^2)(\sigma^2 + \delta^2)}{ \mu^4 (k+k_1-1)} \nonumber \\
      & + \frac{(64+4c_2)  (4\rho^2 N + \chi^2) (\sigma^2 + \delta^2) }{ \mu^2 }. \nonumber
\end{align}

Finally, combining \eqref{ger-21} with Assumption \ref{assumption:Lip}, we obtain

\begin{align}
\label{ger-22}
&   \E \lp F_{\cS}(\bar\vx^{k})  -  F_{\cS} (\vx_{\cS}^*) \rp \\
\leq & \frac{L}{2} \E \|\bar\vx^{k} - \vx_{\cS}^* \|^2  \nonumber\\
\leq &   \frac{L(k_1-1)  \|\bar \vx^0 - \vx_{\cS}^* \|^2}{2(k+k_1-1)} + \frac{4L\sigma^2 ln(k+k_1-1)}{ \mu^2 N (k+k_1-1)} \nonumber \\
& + \frac{4 c_2 L^2 (\sigma^2 + \delta^2)}{\mu^3 (k+k_1-1) }
      + \frac{48 c_2 L^3 (4\rho^2 N+\chi^2) (\sigma^2 + \delta^2)}{ \mu^4 (k+k_1-1)}  \nonumber \\
     & + \frac{(32+2c_1) L  (4\rho^2 N+\chi^2)  (\sigma^2 + \delta^2)}{ \mu^2 }.  \nonumber
\end{align}
This completes the proof.
\end{proof}


\section{Proof of Theorem \ref{the-ge-b}}
\begin{proof}
Similar to the derivation of \eqref{g-21-1}, we have
\begin{align}
    \label{ger-23}
     &  \E \frac{1}{NZ}\sum_{i \in \N} \sum_{z=1}^{Z}  \| \nabla f(\vx^{k}_i; \xi_{i,z})\|^2   \\
    \leq & 4 L^2 \E H^k + 8L \E (F_{\cS}(\bar\vx^{k}) - F_{\cS} (\vx_{\cS}^*)) + 4 \sigma^2 + 4\delta^2. \nonumber
\end{align}

According to the Theorem \ref{the-oe-b} and Lemma \ref{lemma-dm-b},
substituting \eqref{ger-23} into \eqref{ger-13} and \eqref{gb-10-2-1},  if we set $k \asymp \mu N Z$, we have
\begin{align}
\label{ger-24}
      & \frac{1}{NZ}\sum_{i \in \N}\sum_{z=1}^{Z} \E \|\bar\vx'^{k}_{(i,z)} - \bar\vx^{k} \|^2 \\
     \leq & O\lp (4\rho^2 N+\chi^2)  (\sigma^2 + \delta^2) \rp + O\lp  \frac{\sigma^2 +\delta^2 }{ \mu N  Z} \rp \nonumber \\
      & +  O\lp  \frac{(4\rho^2 N+\chi^2)  (\sigma^2 + \delta^2)}{ \mu  N  Z} \rp. \nonumber
\end{align}

According to Lemma \ref{lemma-sta}, let $v = L$ and $k \asymp \mu N Z$. Then by substituting \eqref{ger-22} and \eqref{ger-24} into \eqref{lemma}, we have

\begin{align}
\label{gb-10-9}
    &\E[F(\bar \vx^k)-F_{\cS}(\bar \vx^k)] \\
   \leq &  O(\frac{\|\bar\vx^0 - \vx_{\cS}^*\|^2}{ \mu N Z})+  \tilde O(\frac{\sigma^2}{\mu N Z}) + O(\frac{\delta^2}{ \mu N Z}) \nonumber \\
      & +  O\lp  \frac{(4\rho^2 N+\chi^2)  (\sigma^2 + \delta^2)}{  \mu N Z} \rp \nonumber \\
      & + O\lp (4\rho^2 N+\chi^2)  (\sigma^2 + \delta^2) \rp. \nonumber
\end{align}
This completes the proof.
\end{proof}

\section{Useful Lemmas and Their Proofs}
\subsection{Proof of Lemma 2}
The following lemma characterizes the consensus dynamics among the agents of Algorithm \ref{DSGD} during the learning process. The proof builds upon the approach of Theorem 2 in \cite{wu2022byzantine}, as well as Lemma 2 in \cite{ye2024generalization}.
We define the consensus error as $H^k = \frac{1}{N} \sum_{n \in \N} \| \vx^{k}_n - \bar\vx^{k}\|^2 $, where $\bar\vx^k:= \frac{1}{N}\sum_{n \in \N}\vx^k_n$ is the average of all local models at time $k$.
\begin{lemma}[Consensus Error of Attack-free DSGD]
\label{lemma-dm}
    Under Assumptions \ref{assumption:Lip}--\ref{assumption:heterogeneity}, setting the step size $\alpha^k= \frac{1}{\mu(k+k_0)}$, in which $k_0$ is sufficiently large. When all agents share the same initialization, the consensus error of Algorithm \ref{DSGD} can be bounded by
    \begin{align}
        \label{inequality:H-convergence}
        \E H^k
        \le &  \frac{c_1(\sigma^2 +\delta^2)}{\mu^2 (k+k_0)^2}.
    \end{align}
    Here  $c_1 > 0$ is a constant, and the expectation is taken over all the random variables.
\end{lemma}

\begin{proof}
For convenience, we define $X^{k} = [\vx^{k}_1, \cdots, \vx^{k}_{N}]^\top \in$ $\mathbb{R}^{N \times d}$ and  $X^{k+\frac{1}{2}} = [\vx^{k+\frac{1}{2}}_1, \cdots, \vx^{k+\frac{1}{2}}_{N}]^\top \! \in \mathbb{R}^{N \times d}$. Because $X^{k+1} = W X^{k+\frac{1}{2}}$, we have

 \begin{align}
        \label{dm-1}
         &\sum_{n \in \N}   \|\vx^{k+1}_{n}-\bar\vx^{k+1}\|^2   \\
        = & \|(I-\frac{1}{N}\bm{1}\bm{1}^\top) X^{k+1}\|_{F}^2  \nonumber\\
         \le & \|(I-\frac{1}{N}\bm{1}\bm{1}^\top)W X^{k+\frac{1}{2}}\|_{F}^2
         \nonumber\\
          =&  \|(I-\frac{1}{N}\bm{1}\bm{1}^\top)W (I-\frac{1}{N}\bm{1}\bm{1}^\top) X^{k+\frac{1}{2}}\|_{F}^2
        \nonumber \\
        \le&  \|(I-\frac{1}{N}\bm{1}\bm{1}^\top) W\|^2\|(I-\frac{1}{N}\bm{1}\bm{1}^\top) X^{k+\frac{1}{2}}\|_{F}^2
        \nonumber \\
        =& (1-\lambda)  \| (I-\frac{1}{N}\bm{1}\bm{1}^\top) X^{k+\frac{1}{2}}\|_{F}^2.  \nonumber
    \end{align}

According to Lemma 2 in \cite{wu2022byzantine}, for any $b \in (0,1)$, we have
\begin{align}
\label{g-17}
    \E H^{k+1}
     \leq & (1-\lambda) \lp \frac{1}{1-b}+\frac{6 (\alpha^k)^2 L^2}{b} \rp \E H^k  \\
     & +  (1-\lambda) \frac{4 (\alpha^k)^2 (\sigma^2 + \delta^2)}{b}. \nonumber
\end{align}
If we further set $b =\frac{\lambda}{3}$ and the proper step size satisfying $\alpha^k \leq \frac{1}{3L}\sqrt{(2-\lambda)\lambda^2/(6-2\lambda)}$, we can obtain
\begin{align}
\label{g-18}
    \frac{1}{1-b}+\frac{6 \alpha^2 L^2}{b} \leq 1+\lambda
\end{align}
Substituting \eqref{g-18} into \eqref{g-17}, we have
\begin{align}
    \label{g-19}
    \E H^{k+1} \leq (1-\lambda^2) \E H^k + \frac{12(1-\lambda)}{\lambda} (\alpha^k)^2 (\sigma^2 +\delta^2).
\end{align}

Below we choose the decaying step size  $\alpha^k= \frac{1}{\mu(k+k_0)}$, where $k_0$ is sufficiently large.  By applying telescopic cancellation to \eqref{g-19} from time $0$ to $k$, we deduce that
 \begin{align}
 \label{dm-11}
        &  \E H^{k}
        \le  (1-\lambda^2)^k H^0  \\
         &\hspace{-1em} +\frac{12(1-\lambda)(\sigma^2 +\delta^2)}{\lambda \mu^2} \lp \frac{1}{(k+k_0-1)^2}+\cdot\cdot\cdot+\frac{(1-\lambda^2)^{k-1}}{(k_0)^2} \rp . \nonumber
    \end{align}
When all agents share the same initialization, according to Lemma 5 in \cite{10730755}, if $\frac{(k_0+1)^2}{k_0^2} \leq \frac{2}{2-\lambda^2}$ is satisfied, we have
    \begin{align}
    \label{dm-12}
            &\E H^{k}
        \le   \frac{c_1(\sigma^2 +\delta^2)}{\mu^2(k+k_0)^2}.
    \end{align}
where $c_1 = \frac{24(1-\lambda)}{\lambda^3} $.
\end{proof}

\subsection{Proof of Lemma 3}
The following lemma characterizes the consensus dynamics of Algorithm \ref{robust-DSGD} among the non-Byzantine agents during the learning process. The proof builds upon the approach of Theorem 2 in \cite{wu2022byzantine} and Lemma 2 in \cite{ye2024generalization}.
In this setting, the consensus error is defined as $H^k = \frac{1}{N} \sum_{n \in \N} \| \vx^{k}_n - \bar\vx^{k}\|^2 $, where $\bar\vx^k:= \frac{1}{N}\sum_{n \in \N}\vx^k_n$ is the average of all non-Byzantine models at time $k$.
\begin{lemma}[Consensus of Byzantine-resilient DSGD]
\label{lemma-dm-b}
Suppose that the robust aggregation rules $\{\A_n\}_{n\in \N}$ in Algorithm \ref{robust-DSGD} satisfy Definition \ref{definition:mixing-matrix},  the associated contraction constant satisfies $\rho < \rho^* := \frac{\lambda}{8\sqrt{N}}$ and
Assumptions \ref{assumption:Lip}--\ref{assumption:heterogeneity} hold for all non-Byzantine agents $n \in \N$.
    Set the step size $\alpha^k= \frac{2}{\mu(k+k_1)}$, where $k_1$ is sufficiently large. When all non-Byzantine agents share the same initialization, the consensus error of Algorithm \ref{robust-DSGD} can be bounded by
    \begin{align}
        \label{by-Hk}
        \E H^k
        \le &  \frac{c_2(\sigma^2 +\delta^2)}{\mu^2 (k+k_1)^2}.
    \end{align}
    Here  $c_2 > 0$ is a constant, and the expectation is taken over all the random variables.
\end{lemma}
\begin{proof}
In the presence of Byzantine agents, by Theorem 2 in \cite{wu2022byzantine}, if $\rho < \frac{\lambda}{8\sqrt{N}}$ and we define $w = \lambda - 8\rho \sqrt{N}$, then if $\alpha^k \leq \frac{\sqrt{(2-w)w^2/(6-2w)}}{3L}$, we have
\begin{align}
    \label{bdm-1}
    \E H^k \leq (1-w^2) H^k + \frac{12(1-w)}{w}(\alpha^k)^2 (\sigma^2+\delta^2).
\end{align}

We further choose the decaying step size  $\alpha^k= \frac{2}{\mu(k+k_1)}$, use telescopic cancellation on \eqref{bdm-1} from time $0$ to $k$ and deduce that
 \begin{align}
 \label{bdm-2}
        &\E H^{k}
        \le  (1-w^2)^k H^0  \\
         & \hspace{-1em} +\frac{48(1-w)}{w \mu^2}  (\sigma^2 +\delta^2)\lp \frac{1}{(k+k_1-1)^2}+\cdot\cdot\cdot+\frac{(1-w^2)^{k-1}}{(k_1)^2} \rp . \nonumber
    \end{align}
When all agents share the same initialization, by Lemma 5 in \cite{10730755}, if $\frac{(k_1+1)^2}{k_1^2} \leq \frac{2}{2-w^2}$ is satisfied, we have
    \begin{align}
    \label{bdm-3}
            &\E H^{k}
        \le   \frac{c_2(\sigma^2 +\delta^2)}{\mu^2(k+k_1)^2}.
    \end{align}
where $c_2 = \frac{96(1-w)}{w^3} $.
\end{proof}

\section{Additional Numerical Experiments on CIFAR-10}
\label{sec:cifar10}
In addition to the squared $\ell_2$-norm regularized softmax regression task with a strongly convex loss on the MNIST dataset, we investigate the generalization error of ResNet-18 training on the CIFAR-10 dataset with a nonconvex loss. The CIFAR-10 dataset comprises 60,000 color images across 10 classes, with 50,000 training samples and 10,000 test samples. Data heterogeneity is controlled by assigning training samples to each agent according to a Dirichlet distribution,  $Dir(\beta)$, similar to the setup in the softmax regression task. The step size is set as $\alpha^k = \frac{0.015}{0.01k+1}$.  Other experimental settings, including the network topology, remain consistent with those in the softmax regression task.

\subsection{Experimental Results for Attack-free DSGD}
Fig. \ref{gene-noniid-c} illustrates the generalization error of DSGD under varying levels of data heterogeneity, demonstrating that the generalization error increases with larger data heterogeneity. Fig. \ref{gene-size-c} presents the generalization error of Algorithm \ref{DSGD} for different training sample sizes per agent when $\beta=1$. The numerical results indicate that the generalization error decreases as the training sample size increases. These observations are consistent with those from the softmax regression task on MNIST dataset and support our theoretical findings.

\begin{figure}[t]
\centering
\begin{minipage}{0.49\linewidth}
    \centering
    \includegraphics[width=0.9\linewidth]{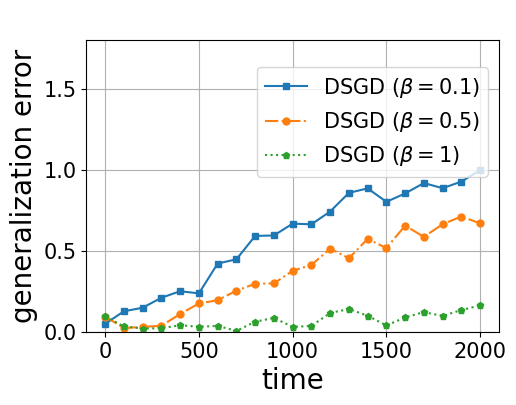}
    \caption{Generalization error of DSGD under different levels of data heterogeneity on CIFAR-10.}
    \label{gene-noniid-c}
\end{minipage}
\hfill
\begin{minipage}{0.49\linewidth}
    \centering
    \includegraphics[width=0.9\linewidth]{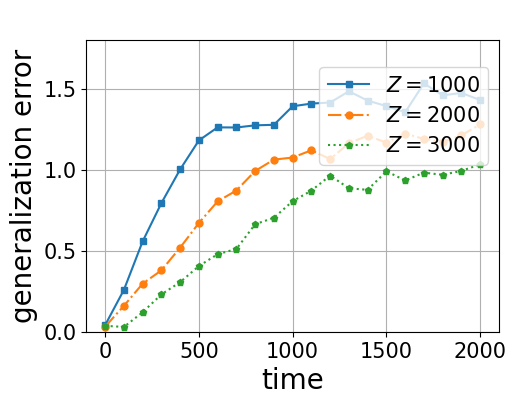}
    \caption{Generalization error of DSGD across different training sample sizes per agent when $\beta=1$ on CIFAR-10.}
    \label{gene-size-c}
\end{minipage}
\end{figure}

\subsection{Experimental Results for Byzantine-resilient DSGD}
To validate the impact of data heterogeneity on the generalization error of Byzantine-resilient DSGD, we present the generalization error of Algorithm \ref{robust-DSGD} equipped with IOS in Fig. \ref{fig:ios-c} under different levels of data heterogeneity. Under the ALIE attack, the generalization error of Algorithm \ref{robust-DSGD} with IOS becomes unstable. Under other attacks, it is obvious that the generalization error increases as $\beta$ decreases.
Furthermore, Fig. \ref{fig:robust-cifar} presents the generalization errors of both attack-free DSGD and Byzantine-resilient DSGD with different robust aggregation rules when $\beta=1$. The results depict that the presence of Byzantine agents enlarges the generalization error. These findings are also consistent with observations from the softmax aggregation task on MINIST dataset and align with our theoretical results.

\subsection{Cooperation Gain of Byzantine-resilient DSGD}
To highlight the cooperation gain in the presence of Byzantine agents, we set the concentration parameter to $\beta = 1$ and randomly select a non-Byzantine agent to train independently on its own dataset without cooperation.  The generalization errors of SGD without cooperation and cooperation-based Byzantine-resilient approaches on CIFAR-10 are illustrated in Fig. \ref{fig:co-c}. We observe that the generalization error of the non-cooperation algorithm is significantly larger than that of the cooperation-based Byzantine-resilient approaches when $\beta = 1$.
This is consistent with observations from the softmax regression task on MNIST and support our theoretical results.

\newpage

\begin{figure*}[htbp]
    \centering
    \includegraphics[width=0.74\linewidth]{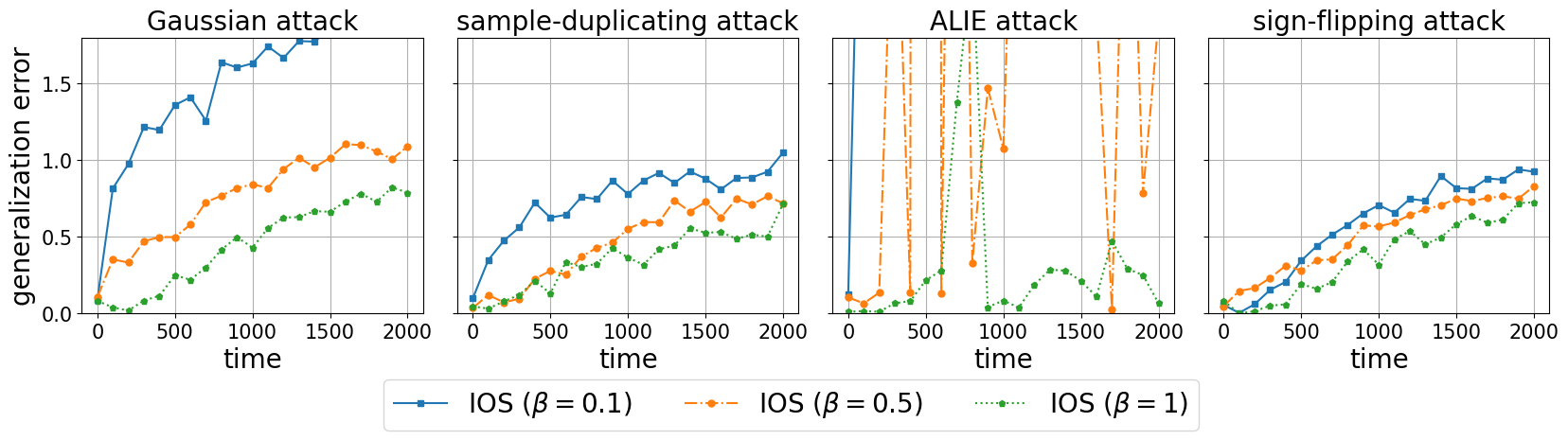}
    \caption{Generalization error of Byzantine-resilient DSGD equipped with IOS under different levels of data heterogeneity on CIFAR-10.}
    \label{fig:ios-c}
\end{figure*}

\begin{figure*}[htbp]
    \centering
    \includegraphics[width=0.74\linewidth]{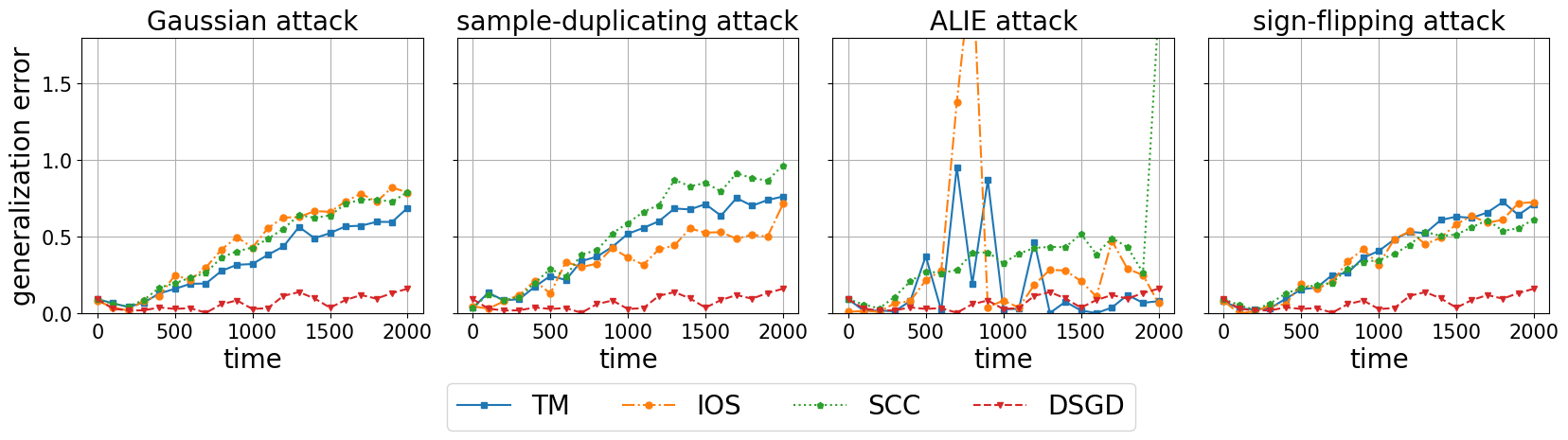}
    \caption{Generalization error of attack-free DSGD and Byzantine-resilient DSGD when $\beta=1$ on CIFAR-10.}
    \label{fig:robust-cifar}
\end{figure*}

\begin{figure*}[htbp]
    \centering
    \includegraphics[width=0.74\linewidth]{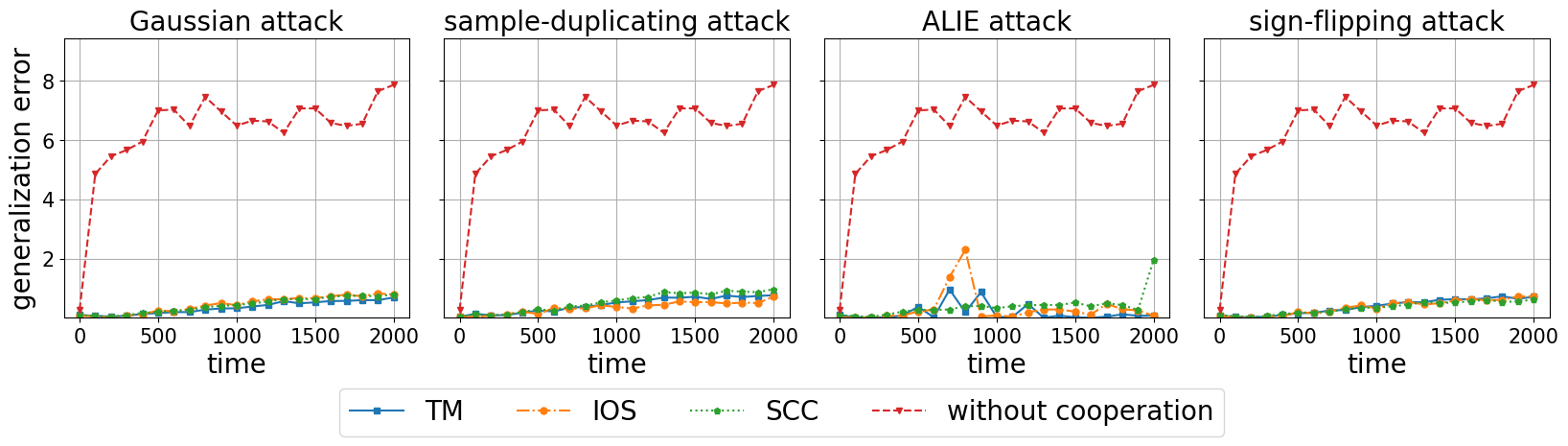}
    \caption{Generalization error of SGD without cooperation and Byzantine-resilient DSGD when $\beta=1$ on CIFAR-10.}
    \label{fig:co-c}
\end{figure*}

\end{appendices}


\end{document}